\documentclass{article}

\PassOptionsToPackage{numbers,compress}{natbib}

 \usepackage[preprint]{neurips_2026}

\usepackage[utf8]{inputenc} %
\usepackage[T1]{fontenc}    %
\usepackage{hyperref}       %
\usepackage{url}            %
\usepackage{booktabs}       %
\usepackage{amsfonts}       %
\usepackage{nicefrac}       %
\usepackage{microtype}      %
\usepackage[dvipsnames]{xcolor}         %

\usepackage{booktabs}
\usepackage{algorithm}
\usepackage{algorithmic}
\usepackage{subcaption}
\usepackage{amssymb}
\usepackage{amsmath}
\usepackage{graphicx}
\usepackage{pifont}
\usepackage{wrapfig}
\usepackage{colortbl}
\usepackage{enumitem}
\usepackage{amsthm}
\usepackage{etoc}

\newtheorem{assumption}{Assumption}
\newtheorem{lemma}{Lemma}
\newtheorem{theorem}{Theorem}

\newcommand{\tablestyle}[2]{\def\arraystretch{#2}\setlength{\tabcolsep}{#1}}

\definecolor{baselinecolor}{gray}{.9}
\newcommand{\baseline}[1]{\cellcolor{baselinecolor}{#1}}

\title{ForgeVLA: Federated Vision-Language-Action Learning without Language Annotations}

\author{%
  Yuhao Zhou$^{1,5}$, Yunpeng Zhu$^{2}$, Yang Zhou$^{1,5}$, Jindi Lyu$^{1,5}$, Jian Lan$^{3}$, \\ 
  \textbf{Zhangyuan Wang$^{2}$, Dan Si$^{4}$, Thomas Seidl$^{3}$, Qing Ye$^{1,5}$, Jiancheng Lyu$^{1,5}$} \\
  Sichuan University$^{1}$, Zhejiang University$^{2}$, \\
  Ludwig-Maximilians-Universität München$^{3}$, Lenovo Group Limited$^{4}$, \\
  Engineering Research Center of Machine Learning and Industry Intelligence, Ministry of Education$^{5}$\\
  \texttt{sooptq@gmail.com}, \texttt{lvjiancheng@scu.edu.cn} \\
}

\begin{document}

\maketitle
\etocsettocdepth.toc{none}

\begin{abstract}
    Vision-Language-Action (VLA) models hold great promise for general-purpose robotic intelligence, yet scaling up such models is severely bottlenecked by the high cost of acquiring annotated training data.
Fortunately, vision-equipped robots deployed across various domains already produce abundant vision-action pairs that can be leveraged to scale up VLA training more efficiently.
However, these raw data cannot be centrally aggregated due to various constraints and also exhibit severe heterogeneity.
To address these challenges, in this paper, we propose ForgeVLA, a federated VLA training framework that learns VLA models from distributed vision-action pairs without centralizing raw data or requiring manual annotations.
Specifically, each client in ForgeVLA is equipped with an embodied instruction classifier that maps vision-action pairs to a predefined instruction set, recovering the missing language modality and forming complete vision-language-action triplets.
Beyond triplet construction, we also identify vision-language feature collapse as a critical challenge that has been largely overlooked in prior federated VLA research.
To mitigate this issue, ForgeVLA combines a client-side contrastive planning loss with a server-side adaptive aggregation strategy to learn task-discriminative representations efficiently.
Extensive experiments across multiple benchmarks show that ForgeVLA significantly outperforms other baselines, and ablation studies further validate the contribution of each component.

\end{abstract}

\section{Introduction}

Vision-Language-Action (VLA) models integrate visual perception, language understanding, and motor control into a single policy, advancing toward general-purpose robotic intelligence~\cite{brohan2023rt2, driess2023palme, black2024pi0, kim2024openvla}.
As with LLMs and VLMs~\cite{kaplan2020scaling, hoffmann2022chinchilla, muennighoff2023scaling}, VLA models benefit from larger and more heterogeneous training corpora, with stronger cross-embodiment transfer observed at greater data scale~\cite{walke2023bridgedata, khazatsky2024droid, openx2024, zheng2026egoscale}.
However, unlike text or image-text pairs that can be collected at near-zero marginal cost, VLA training samples typically require physical execution, making data collection a central bottleneck for scaling VLA capabilities~\cite{dasari2019robonet, mandlekar2023mimicgen}.

Prior work has explored multiple directions to alleviate data scarcity in VLA training:
data augmentation~\cite{laskin2020reinforcement, kostrikov2020image, mandlekar2023mimicgen} provides limited additional diversity;
simulation-based pipelines~\cite{tobin2017domain, james2020rlbench, mu2021maniskill} are constrained by the reality gap;
and generative models that synthesize visual trajectories~\cite{du2024learning, black2024zeroshot, ko2023learning} still struggle to produce physically plausible action labels.
These approaches all expand the training corpus \emph{synthetically}, raising a natural question: \emph{can we source high-quality VLA training data at scale and low cost without relying on synthetic generation}?

Historically, LLM and VLM scaling benefited from repurposing existing artifacts at low cost~\cite{villalobos2022will} (\textit{e.g.}, decades of written text~\cite{brown2020language,touvron2023llama}).
A similar opportunity may exist in robotics~\cite{kehoe2015survey} as well: millions of vision-equipped robots deployed in manufacturing~\cite{ifr2025}, warehousing~\cite{amazon2024robots}, healthcare~\cite{intuitive2024}, autonomous driving~\cite{sun2020scalability}, and other domains~\cite{duckett2018agricultural, abolhasani2023rise, wynn2014autonomous} could already log synchronized visual observations and action trajectories suitable for VLA training.
The missing modality is natural-language task description, which is typically unnecessary at deployment time and difficult to recover post hoc because a continuous action stream admits multiple valid descriptions at different granularities.
Thus, the bottleneck for scaling VLA may largely reduce to the absence of language annotations for otherwise valuable vision–action logs.

Closing this annotation gap at scale faces two tightly coupled challenges:
(1) \emph{Privacy}: robotic data is siloed within independent organizations (\textit{e.g.}, factories and hospitals) under strict confidentiality constraints.
(2) \emph{Heterogeneity}: clients differ in embodiment, task distributions, sensing conditions, and environments, inducing severely non-independent and identically distributed (non-i.i.d.) data.

Federated learning (FL)~\cite{mcmahan2017fedavg, kairouz2021advances} provides a natural framework for this setting, allowing clients to exchange only model updates.
However, existing federated VLA methods leave two issues unresolved:
(1) \emph{Data}: FedVLA~\cite{cui2025fedvla} and FLAME~\cite{bou2025flame} assume fully annotated VLA triplets at each client, offering no mechanism to incorporate the far more abundant unannotated vision–action logs.
(2) \emph{Heterogeneity}: we empirically find that standard non-i.i.d. mitigation strategies (\textit{e.g.}, FedProx~\cite{li2020fedprox}) do not transfer to VLA training and can even degrade performance (Table~\ref{tab:fedprox-exp}), suggesting a failure mode beyond conventional client drift.

To address these challenges, we propose \textbf{ForgeVLA}, a federated framework that repurposes legacy vision–action logs for VLA training without centralizing raw data or manually annotating task instructions.
On the data side, ForgeVLA equips each client with an \emph{embodied instruction classifier}, a pretrained VLM fine-tuned on a small public VLA dataset, that maps local vision-action pairs to a predefined instruction set entirely on-device, recovering the missing language modality.
On the heterogeneity side, we trace the root cause to vision-language feature collapse and introduce (i) a contrastive planning loss that promotes task-discriminative representations during local training, and (ii) an adaptive aggregation strategy that preserves client update directions on the server.
In summary, our main contributions are:

\vspace{-5pt}
\begin{enumerate}[leftmargin=1.2em]
\item We identify the installed base of vision-equipped robots as a largely untapped opportunity to scale VLA training beyond curated or synthetic datasets.
Because privacy constraints preclude centralized access, we propose ForgeVLA, a federated paradigm that recovers the missing language modality on-device via classification over a predefined instruction set and converts distributed legacy logs into usable VLA training data.
\item We show that prior methods fail to address heterogeneity in federated VLA training, and we trace the root cause to vision-language feature collapse under cross-domain distributional shifts. 
Motivated by this finding, we mitigate collapse with a contrastive planning loss for task-invariant local representation learning and a adaptive aggregation strategy on the server.
\item Extensive experiments show that ForgeVLA significantly outperforms federated baselines across multiple benchmarks. 
Source code will be released after publication.
\end{enumerate}

\section{Related Work}

\textbf{Vision-Language-Action Models}:
Vision-Language-Action (VLA) models provide an end-to-end paradigm for generalist robot control by integrating visual perception, language understanding, and action prediction within a single architecture.
RT-2~\citep{brohan2023rt2} pioneered this direction, and the open-sourced OpenVLA line~\citep{kim2024openvla,kim2025openoft} further broadened its adoption while improving task performance and inference efficiency.
Subsequent work has advanced VLA models along several axes, with many reporting state-of-the-art results on standard robotic manipulation benchmarks (\textit{e.g.}, dexterous manipulation, 3D spatial reasoning, action-prediction module design, and training frameworks~\citep{black2024pi0,li2024cogact,qu2025spatialvla,chen2025internvla,wen2025dexvla}).
In parallel, multi-embodiment learning aims to train policies that generalize across heterogeneous robot platforms without manual action-space alignment~\citep{ghosh2024octo,doshi2024crossformer,wang2024hpt}.
Despite these advances, current VLA methods still rely on large-scale, high-quality annotated data, whose manual labeling cost and limited scalability remain a central bottleneck for robot learning.

\textbf{Federated Vision-Language-Action Learning}:
Federated learning (FL)~\citep{mcmahan2017fedavg, zhou2021communication} enables collaborative training across distributed clients without sharing raw data.
Specifically, each client in FL optimizes locally and transmits only model updates to a central server for global aggregation~\cite{zhou2023communication, zhou20253sfc}.
A central challenge in FL is client heterogeneity.
When client data are non-i.i.d., local updates can drift from the global optimum and slow or destabilize convergence.
Accordingly, a range of methods~\citep{li2020fedprox, karimireddy2020scaffold, wang2020fednova, acar2021feddyn, shi2023prior, zhou2024federated, zhou2025deploying} mitigate non-i.i.d. effects via mechanisms such as proximal regularization, control variates, and client clustering.
Privacy-enhancing techniques, including secure aggregation~\citep{bonawitz2017practical} and differential privacy~\citep{abadi2016deep}, can be layered on top of standard FL to provide formal privacy guarantees, though prior gradient-leakage analyses~\citep{zhu2019deep, geiping2020inverting} show that unprotected model updates can reveal training data.

Recently, federated VLA has gained attention because robotic data are generated in a highly distributed manner and are often difficult to centralize due to privacy regulations and communication constraints.
In this line, FLAME~\citep{bou2025flame} introduces an FL benchmark for robotic manipulation with over 160K demonstrations and evaluation protocols.
FedVLA~\citep{cui2025fedvla} proposes an FL framework for VLA, incorporating instruction-oriented scene parsing, a dual-gating mixture-of-experts architecture with token–expert joint routing, and expert-aware server aggregation.
Unlike prior work that assumes fully annotated VLA triplets at each client, ForgeVLA targets a more realistic and challenging setting in which clients only store unannotated vision–action logs.
Notably, ForgeVLA recovers the missing language modality on-device via classification over a predefined instruction set and addresses the vision-language feature collapse that arises under cross-domain federated VLA training.

\section{Problem Formulation}

We consider a standard FL setup consisting of one central server and $N$ participating clients, indexed by $i \in [N] \triangleq \{1, 2, \dots, N\}$. 
Each client $i$ holds a private local dataset $D_i$ containing $K_i$ time-synchronized vision--action pairs, collected from its embodied robotic system:
\begin{equation}
D_i = \left\{ (v_i^k, a_i^k) \right\}_{k=1}^{K_i},
\end{equation}
where $v_i^k \in \mathcal{V}$ denotes a visual observation sampled from the vision space $\mathcal{V}$, and $a_i^k \in \mathcal{A}$ denotes a corresponding robot action sampled from the action space $\mathcal{A}$. 
A defining characteristic of our problem is that the local dataset $D_i$ lacks language instructions $l_i^k \in \mathcal{L}$ due to unnecessity and ambiguity during deployments, where $\mathcal{L}$ is the language space and there are $M$ distinct instructions in total.

The global VLA model follows an encoder-decoder architecture, parameterized by $\theta = (\theta_{\text{enc}}, \theta_{\text{dec}}) \in \Theta_{\text{vla}}$. 
The vision-language encoder $f_{\theta_{\text{enc}}}: \mathcal{V} \times \mathcal{L} \rightarrow \mathcal{Z}$ maps a visual observation and a natural language instruction to a joint latent representation $z \in \mathcal{Z}$, where $\mathcal{Z}$ is the latent embedding space. 
The action decoder $f_{\theta_{\text{dec}}}: \mathcal{Z} \rightarrow \mathcal{A}$ takes the joint latent representation $z$ and predicts the corresponding robot action. 
The overall VLA mapping is thus the composition: $f_\theta(v, l) = f_{\theta_{\text{dec}}}(f_{\theta_{\text{enc}}}(v, l))$.
To bridge the missing language modality locally, each client $i$ owns a pretrained instruction classifier $c_{\phi}: \mathcal{V} \times \mathcal{A} \rightarrow \mathcal{L}$, parameterized by $\phi \in \Phi$. 
This classifier generates a plausible language instruction $\hat{l}_i^k = c_{\phi}(v_i^k, a_i^k)$ from a local vision--action pair.
Let $\ell(\cdot, \cdot)$ denote a task-specific loss function, the local empirical loss on client $i$ can be defined as:
\begin{equation}
\mathcal{L}_i(\theta; D_i, \phi) = \frac{1}{K_i} \sum_{k=1}^{K_i} \ell\left( f_{\theta_{\text{dec}}}\left(f_{\theta_{\text{enc}}}\left(v_i^k, c_{\phi}(v_i^k, a_i^k)\right)\right), a_i^k \right).
\label{eq:vla-loss}
\end{equation}
Then, the objective of federated VLA learning is to find the optimal global VLA parameters $\theta = (\theta_{\text{enc}}, \theta_{\text{dec}})$ that minimize the weighted sum of all local empirical losses, while adhering to the aforementioned privacy and heterogeneity constraints:
\begin{equation}
\arg\min_{\theta \in \Theta_{\text{vla}}} \sum_{i=1}^N w_i \cdot \mathcal{L}_i(\theta; D_i, \phi),
\end{equation}
where $w_i \geq 0$ are client-specific weights satisfying $\sum_{i=1}^N w_i = 1$ that helps combine clients' local losses into a global optimization target.
The most widely adopted $w_i$ is proportional to the local dataset size, \textit{i.e.}, $w_i = K_i / \sum_{j=1}^N K_j$.

\begin{figure*}[tb]
    \centering
    \includegraphics[width=\textwidth]{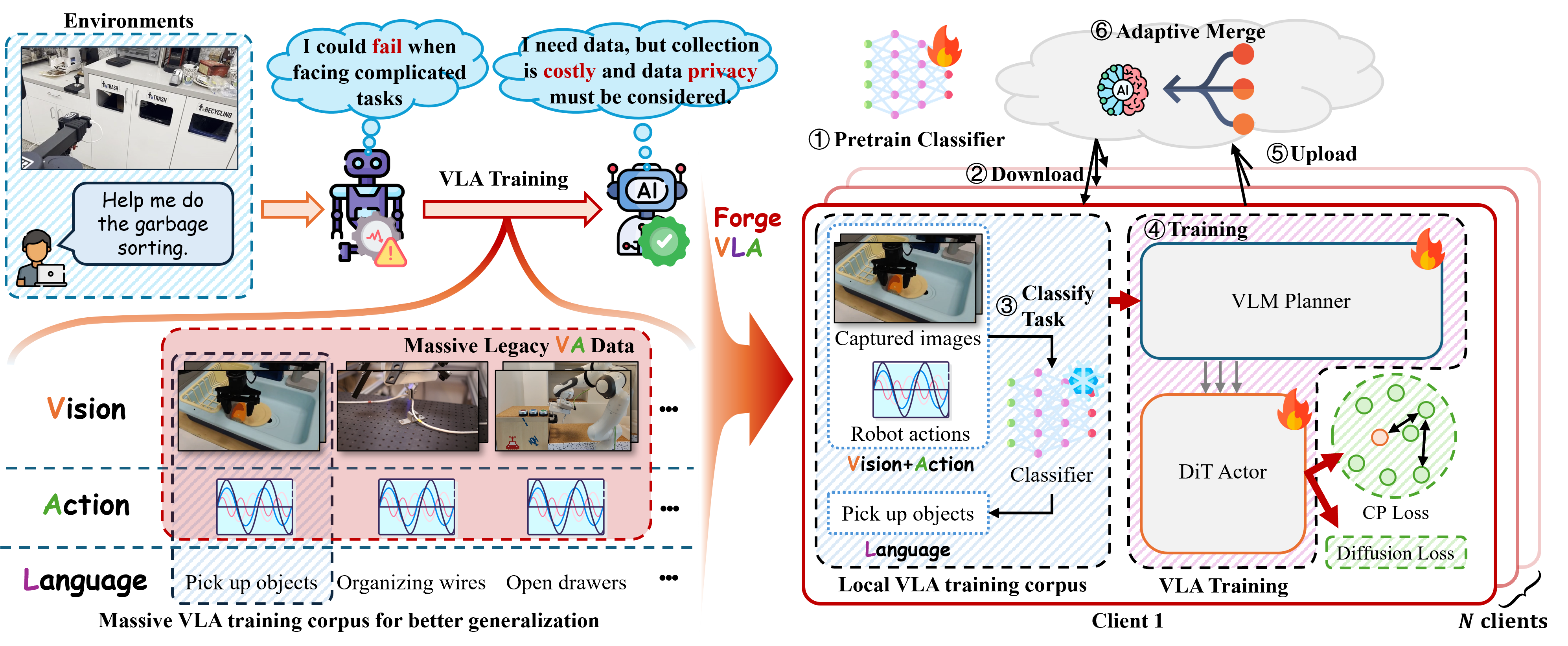}
    \caption{
    \textbf{[Left]}: The key bottleneck for scalable VLA training is data scarcity, as high-quality annotated VLA data are limited, while large volumes of vision–action logs remain underutilized.
    \textbf{[Right]}: ForgeVLA across $N$ clients: \ding{172} Train an embodied instruction classifier on the central server; \ding{173} Clients download the pretrained classifier and the initialized global VLA model; \ding{174} Perform on-device task classification to generate language annotations and construct a complete VLA training corpus; \ding{175} Conduct local VLA training with the task loss and the contrastive planning loss; \ding{176} Upload model updates to the server; \ding{177} The server performs adaptive aggregation to update the global VLA model. 
    This framework enables scalable VLA training from real-world robot logs at near-zero marginal cost while preserving privacy and mitigating data heterogeneity.
    }
    \label{fig:main-arch}
\vspace{-10pt}
\end{figure*}

\section{ForgeVLA}

Figure~\ref{fig:main-arch} overviews the ForgeVLA architecture.
On the server, we fine-tune a pretrained VLM on a small public VLA dataset to obtain an embodied instruction classifier $c_{\phi}$. 
The server then broadcasts $c_{\phi}$ together with the initialized global VLA model $f_\theta$ to all clients.
\begin{wrapfigure}{r}[0pt]{0.4\linewidth}
    \centering
    \vspace{-10pt}
    \includegraphics[width=\linewidth]{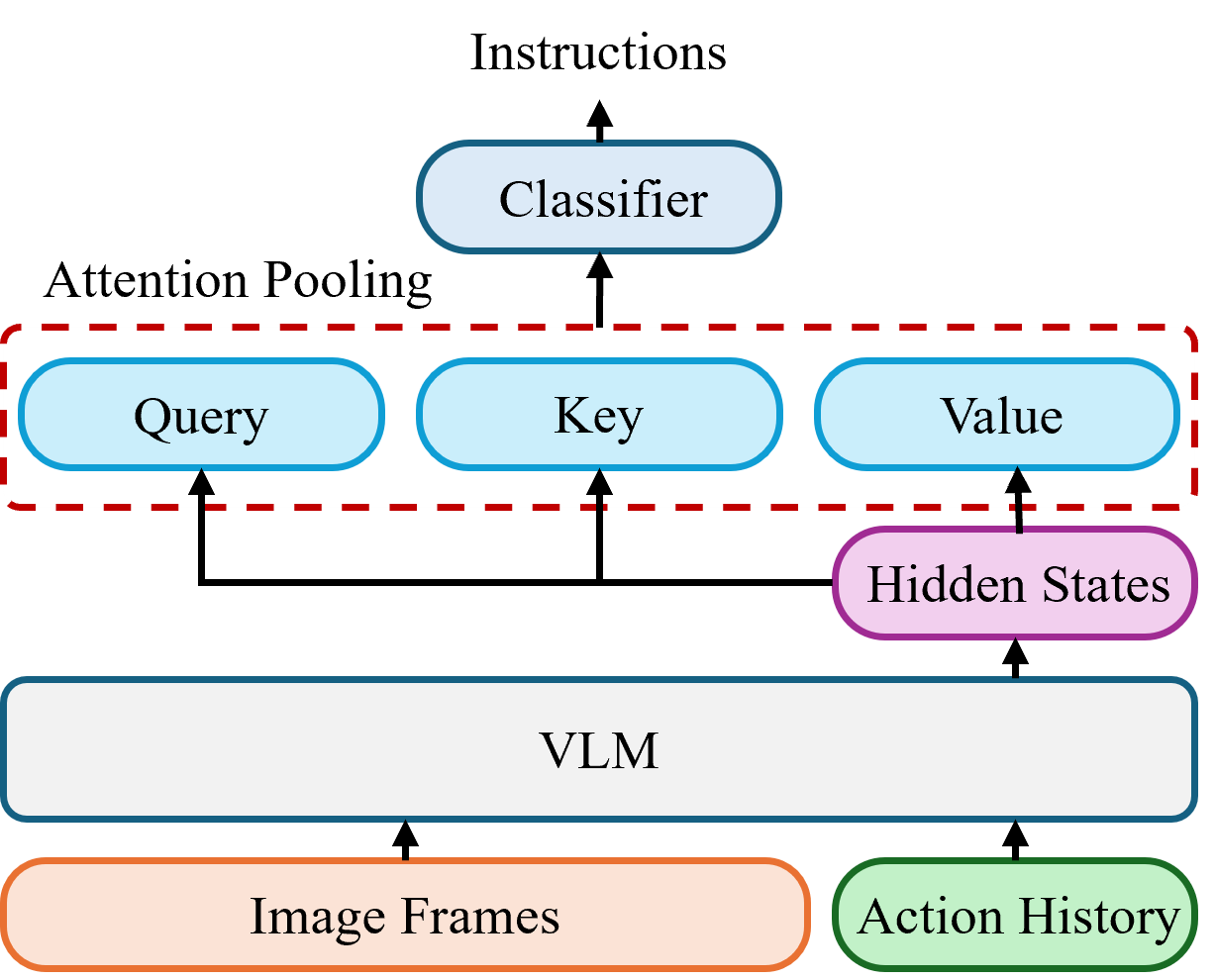}
    \caption{The Instruction Classifier.}
    \label{fig:classifier-arch}
    \vspace{-10pt}
\end{wrapfigure}
On each client, $c_{\phi}$ classifies local vision–action logs into the predefined instruction set, forming a VLA training corpus while keeping all raw data on-device to preserve privacy. 
The client then performs local VLA training with a contrastive planning loss to counteract heterogeneity-induced degradation.
After local training, clients upload only model updates, and the server applies adaptive aggregation to reduce cross-client update conflicts.
Overall, ForgeVLA addresses three bottlenecks for scaling VLA: leveraging underutilized real-world vision–action data, avoiding centralized collection of sensitive logs, and mitigating client heterogeneity.

\subsection{Embodied Instruction Classifier}

To recover the missing task description in raw vision–action logs and construct a valid local VLA training corpus, we propose an embodied instruction classifier built on a pretrained VLM.
The classifier operates over a predefined set of $M$ task instructions, which is a practical design choice that covers the structured deployment scenarios common in industrial and service robotics.
Figure~\ref{fig:classifier-arch} summarizes the design: 
we augment the VLM backbone with a lightweight attention-pooling head and a instruction classification layer.
Given a vision–action pair, the pretrained VLM produces final-layer hidden states. 
The attention-pooling module computes weights over these states and aggregates them into a context vector via weighted summation. 
The classifier then maps the context vector to logits over the instruction set. 
Leveraging the pretrained VLM’s strong generalization, this classifier requires only lightweight fine-tuning on a small public VLA dataset (as shown in Section~\ref{sec:scaling-analysis-classifier}), enabling efficient and robust on-device annotation.

\begin{figure}
    \centering
    \begin{subfigure}[tb]{0.325\linewidth}
        \centering
        \includegraphics[width=\linewidth]{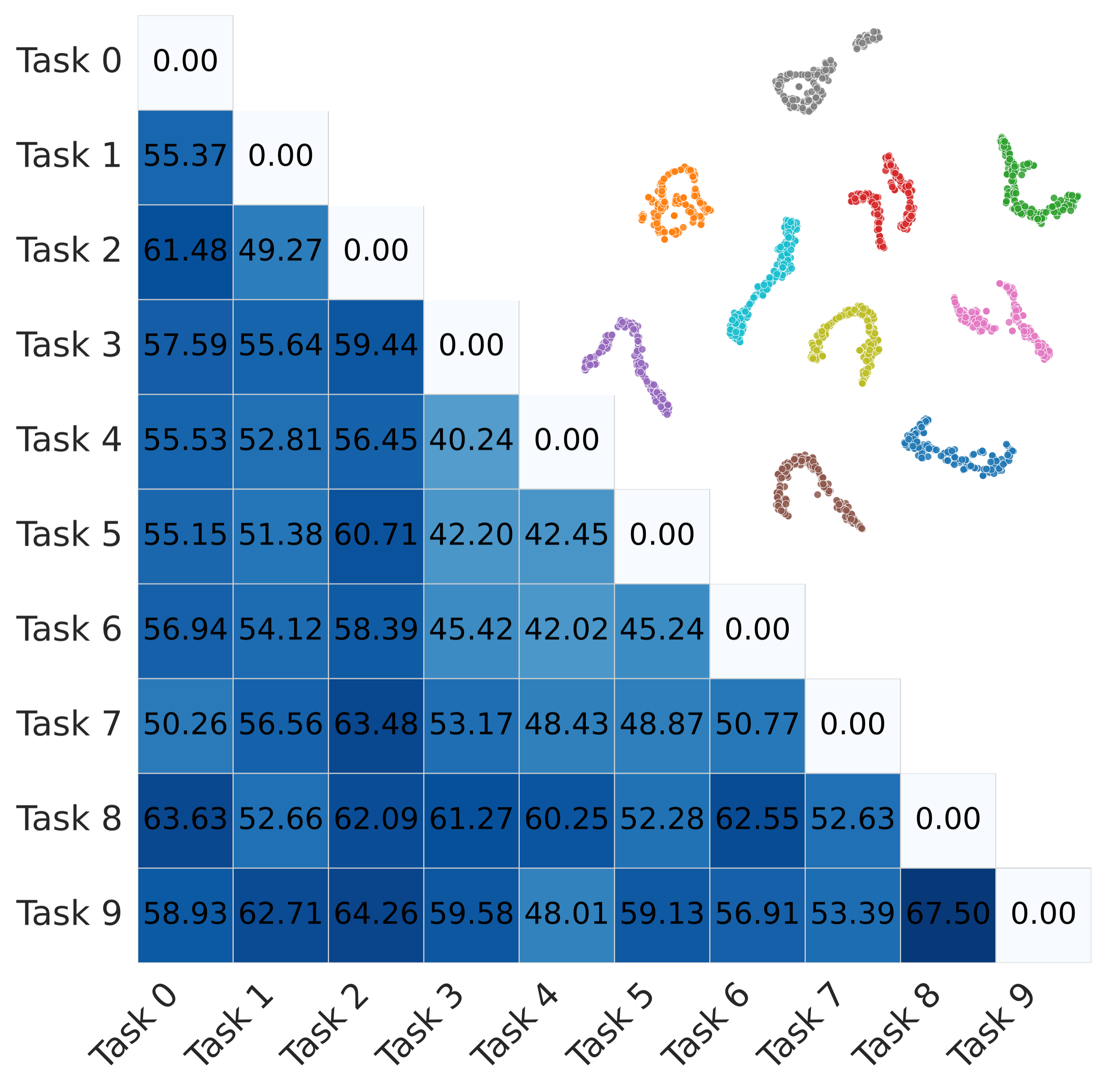}
        \caption{Centralized}
    \end{subfigure}
    \hfill
    \begin{subfigure}[tb]{0.325\linewidth}
        \centering
        \includegraphics[width=\linewidth]{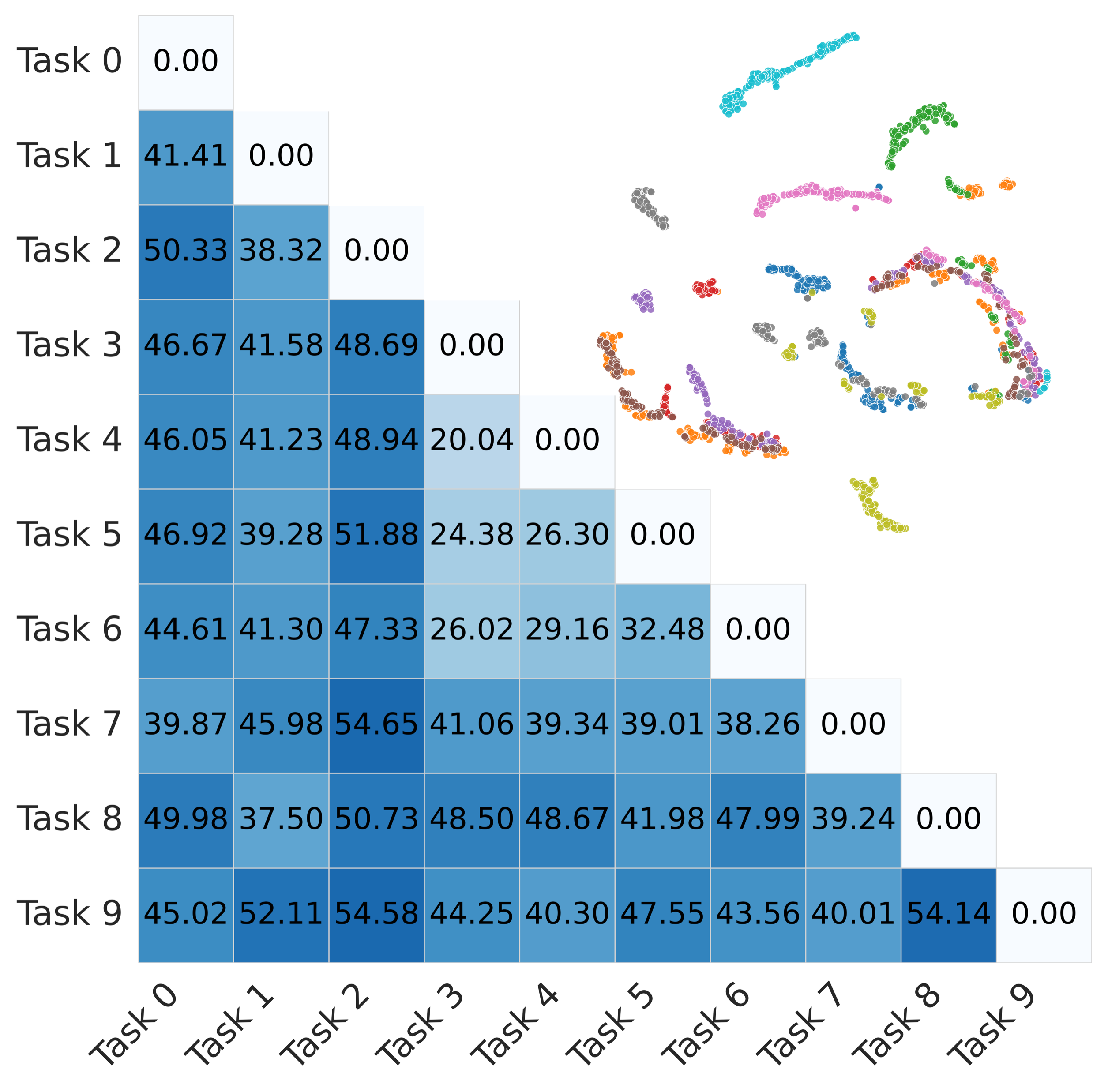}
        \caption{FedAvg}
    \end{subfigure}
    \hfill
    \begin{subfigure}[tb]{0.325\linewidth}
        \centering
        \includegraphics[width=\linewidth]{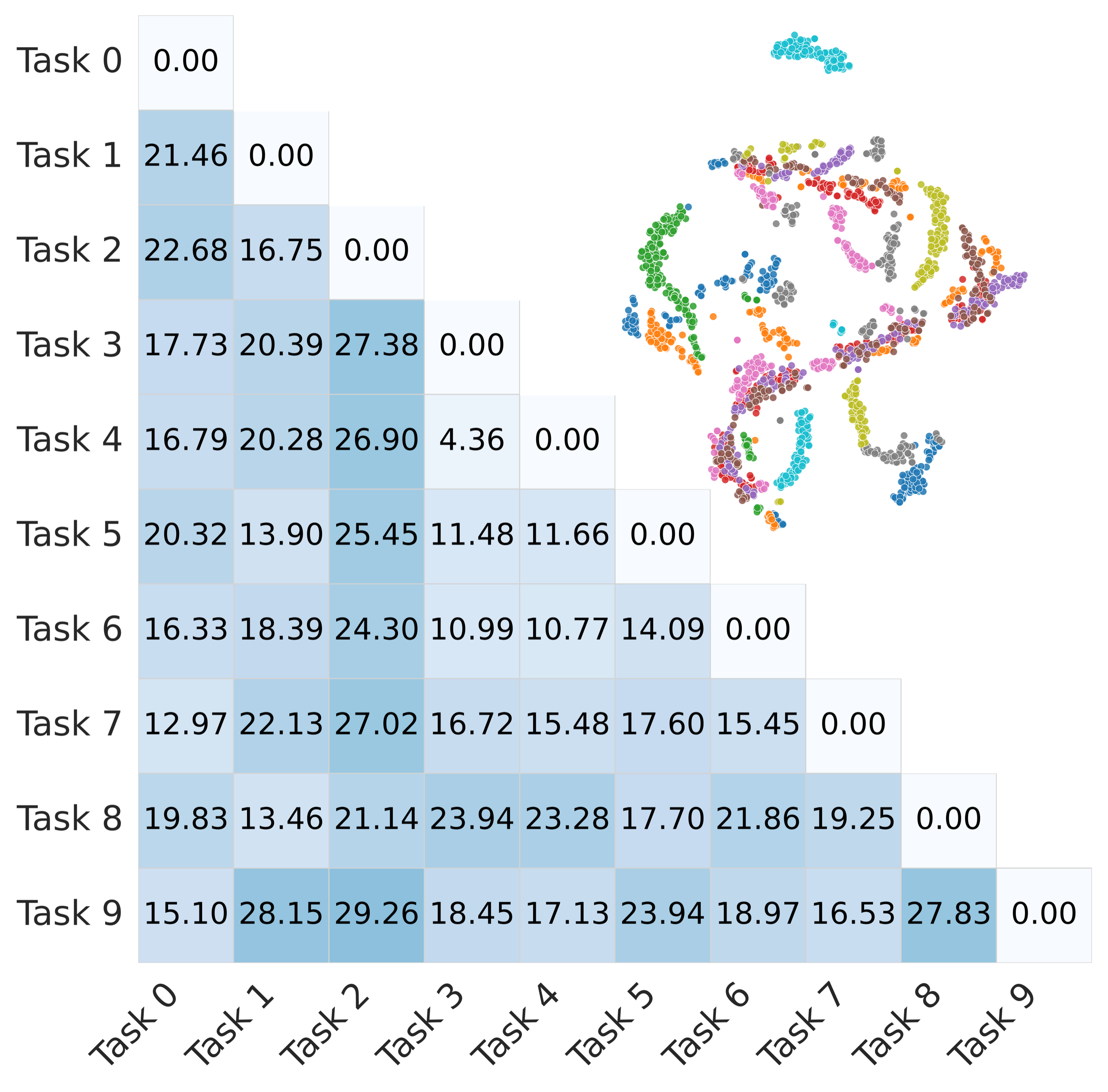}
        \caption{FedProx}
    \end{subfigure}
    \caption{The illustrated distances and T-SNE~\cite{van2008visualizing} projections of latent feature representations learned by different models across different tasks.}
    \label{fig:vl-feature-collapse}
\vspace{-10pt}
\end{figure}

\subsection{Vision-Language Feature Collapse}

Real-world robotic deployments exhibit pronounced cross-client heterogeneity. 
For example, manufacturing robots operating in different factories often share only a small subset of overlapping tasks, yielding highly skewed and largely disjoint task distributions. 
This issue is further amplified in FL because privacy constraints preclude centralizing raw data to smooth such heterogeneity.

Although numerous FL methods target non-i.i.d. data, we observe that directly applying them to federated VLA can severely degrade performance and even destabilize training, even when each client performs only a single local optimization step (\textit{i.e.}, classical client drift is largely eliminated). 
In preliminary experiments (Table~\ref{tab:fedprox-exp}), FedProx~\citep{li2020fedprox} performs substantially worse than vanilla FedAvg~\citep{mcmahan2017fedavg} and sometimes fails to converge, indicating that existing non-i.i.d. remedies are insufficient for federated VLA.

To diagnose the cause, we conduct a systematic analysis on Libero-Goal~\cite{liu2023libero}.
Specifically, we compare three models:
(i) a centralized upper-bound model trained on aggregated data directly, (ii) FedAvg with each client holding data from three tasks, and (iii) FedProx under the same client partition.
We then extract latent feature representations and visualize their pairwise distances across tasks.
As shown in Figure~\ref{fig:vl-feature-collapse}, relative to centralized training, both FedAvg and FedProx yield tightly clustered task embeddings with small inter-task margins, indicating reduced ability to discriminate among distinct manipulation goals.
FedProx exacerbates this effect, consistent with FedProx's proximal term pulling local models toward a global model whose representations are already collapsed.
These results reveal an underexplored failure mode in federated VLA: \emph{vision–language feature collapse}, where heterogeneous federated training causes task-specific vision-language representations to lose discriminability.
Notably, this mode is distinct from the conventional client drift studied in prior FL work, because it persists even with a single local step ($P=1$, \textit{i.e.}, no client drift).
We note that while we visualize collapse on Libero-Goal, the consistently large improvements of ForgeVLA over FedAvg on all four LIBERO benchmarks (Table~\ref{tab:main-results}) suggest the phenomenon is general across task distributions.
To mitigate this collapse, we propose a dual-strategy framework with complementary client-side and server-side components:

\textbf{Contrastive Planning Loss}:
A direct way to mitigate vision–language feature collapse is to enforce discriminative margins between task representations in the training objective.
To this end, we introduce a contrastive planning loss built on a global task representation bank $\{ u_1, u_2, \dots, u_M \}$ maintained by the central server.
At the start of communication round $e$, the server broadcasts the current normalized bank $U^e = \{ \hat{u}_1^e, \hat{u}_2^e, \dots, \hat{u}_M^e \}$ to all clients, providing consistent anchors for local contrastive learning.
For each local sample $(v_i^k, a_i^k)$, let $\hat{l}_i^k = l_{t_i^k}$ denote the instruction predicted by $c_\phi$, let $z_i^k = f_{\theta_{\text{enc}}}(v_i^k, \hat{l}_i^k)$, and let $\hat{z}_i^k = z_i^k / \|z_i^k\|_2$.
During local training, client $i$ computes the contrastive planning loss as follows:
\begin{equation}
\mathcal{L}_{\text{CP},i}^e
=
- \frac{\alpha_{\text{CP}}}{K_i}
\sum_{k=1}^{K_i}
\left(
\frac{(\hat{z}_i^k)^\top \hat{u}_{t_i^k}^e}{\tau}
-
\log \sum_{m=1}^M \exp\left( \frac{(\hat{z}_i^k)^\top \hat{u}_m^e}{\tau} \right)
\right),
\label{eq:cp-loss}
\end{equation}
where $t_i^k \in [M]$ is the instruction index associated with $\hat{l}_i^k$, $\tau$ is a fixed temperature hyperparameter, and each $\hat{u}_m^e$ is $\ell_2$-normalized.
Finally, $\mathcal{L}_{\text{CP},i}^e$ is added to the local VLA training objective.
After local training, each client computes the updated mean latent representation for each of its local tasks and uploads these task embedding to the server. 
The server then updates the global task representation bank via weighted aggregation, renormalizes each bank entry, and broadcasts the refreshed bank in the next communication round. 
This closed loop maintains discriminative task representations throughout federated training.

\begin{wrapfigure}{r}[0pt]{0.45\linewidth}
    \centering
    \vspace{0pt}
    \begin{minipage}[tb]{0.49\linewidth}
        \begin{subfigure}[tb]{\linewidth}
            \centering
            \includegraphics[width=\linewidth]{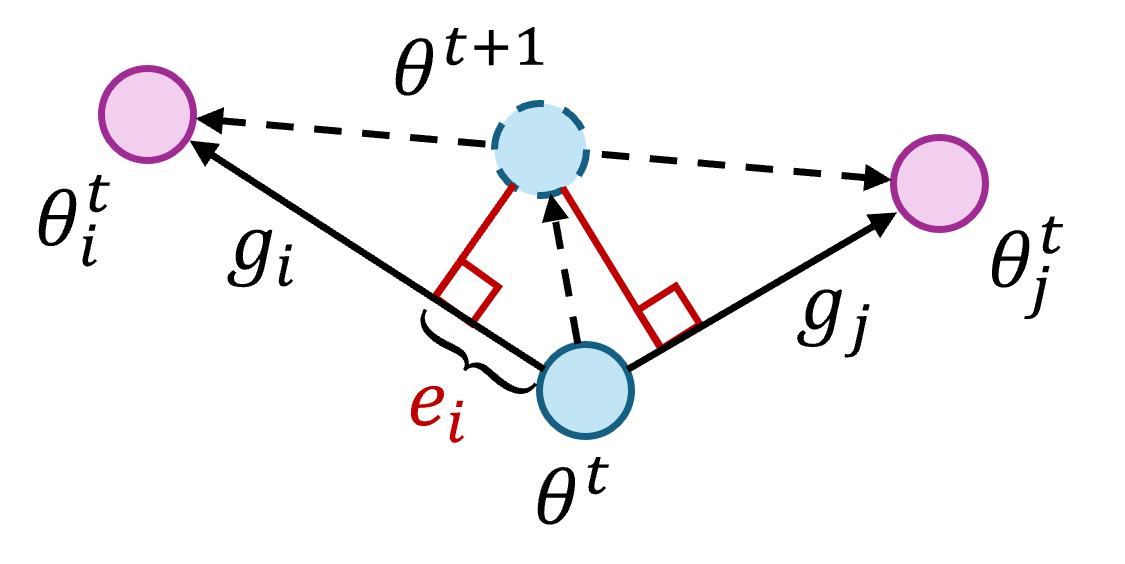}
            \caption{Simple Averaging}
        \end{subfigure}
        
        \begin{subfigure}[tb]{\linewidth}
            \centering
            \includegraphics[width=\linewidth]{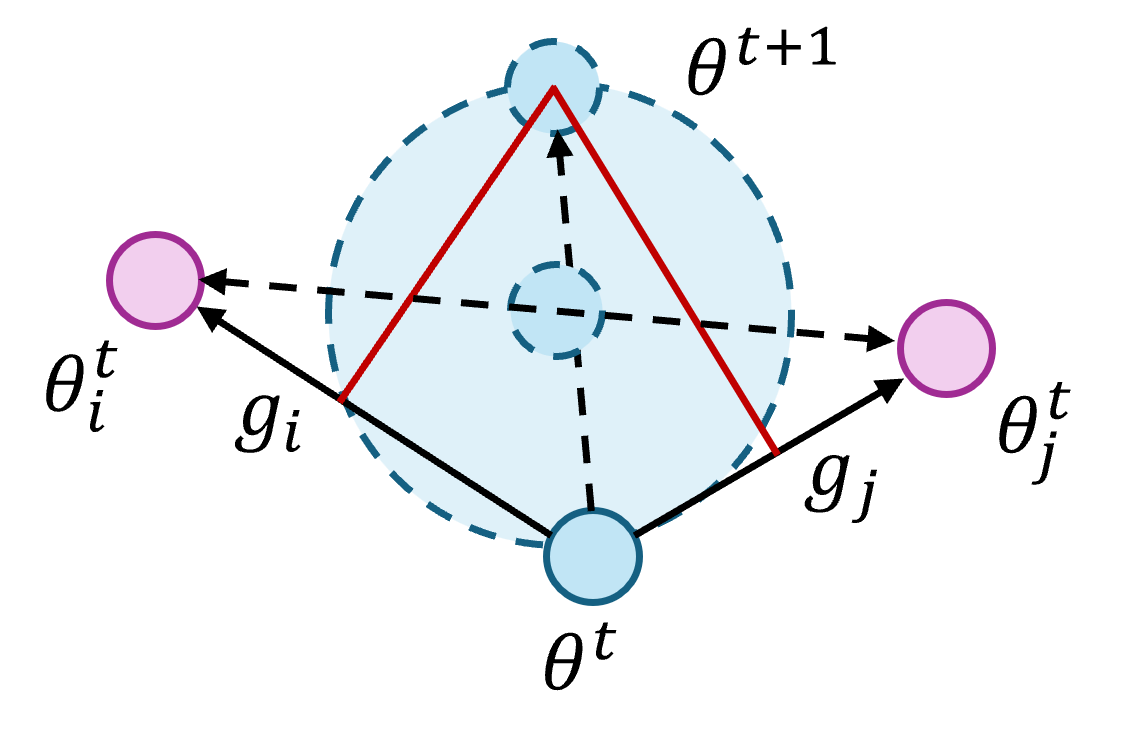}
            \caption{Constrained Adaptive Aggregation}
        \end{subfigure}
    \end{minipage}
    \hfill
    \begin{subfigure}[tb]{0.49\linewidth}
        \centering
        \vspace{16pt}
        \includegraphics[width=\linewidth]{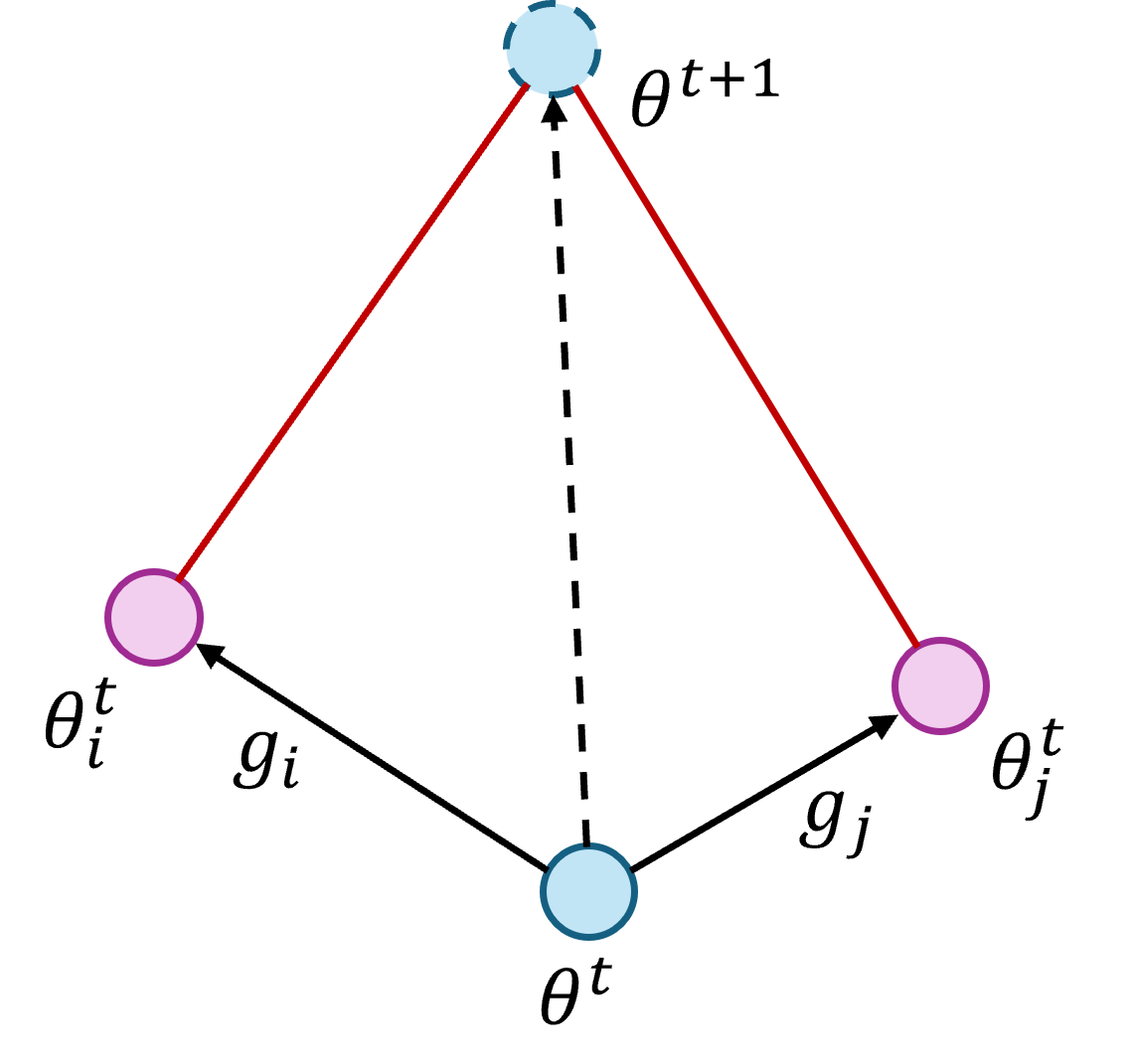}
        \vspace{8pt}
        \caption{Unconstrained Adaptive Aggregation}
    \end{subfigure}
    \caption{Illustrations of Aggregations.}
    \label{fig:adaptive-aggregation}
    \vspace{-15pt}
\end{wrapfigure}

\textbf{Adaptive Aggregation Strategy}:
As shown in Figure~\ref{fig:adaptive-aggregation} (a), simple averaging of heterogeneous client updates can cancel conflicting directions, producing a global update whose projection onto each client’s update direction is substantially attenuated. 
This cancellation slows convergence and degrades performance.
Based on this observation, we propose a server-side adaptive aggregation strategy that solves a data-free optimization problem using only client updates.
At communication round $e$, let $S^e \subseteq [N]$ denote the clients that upload local models, let $\theta_i^{e,P}$ denote client $i$'s model after $P$ local optimization iterations, let $g_i^e = \theta_i^{e,P} - \theta^e$ denote the uploaded update, and let $\tilde{w}_i^e$ denote the server-side weight assigned to an uploaded update.
We encourage the global update $\theta^{e+1} - \theta^e$ to match each uploaded client update $g_i^e$ in projection by minimizing the deviation of the projection coefficient from 1, \textit{i.e.},
\begin{equation}
    \arg\min_{\theta^{e+1}} \sum_{i \in S^e} \tilde{w}_i^e \left\| \frac{(\theta^{e+1} - \theta^e)^\top g_i^e}{\|g_i^e\|_2^2 + \varepsilon_{\text{AG}}}  - 1 \right\|_2^2,
\end{equation}
where $\varepsilon_{\text{AG}} > 0$ is a small numerical stabilizer when a client update is nearly zero.
While this objective promotes per-client update preservation, it can become unstable when client updates are highly conflicting (Figure~\ref{fig:adaptive-aggregation} (c)).
For example, if two clients’ updates are nearly anti-parallel, the unconstrained objective may drive $\|\theta^{e+1} - \theta^e\|_2$ to grow without bound.
To stabilize training, we regularize the solution by anchoring $\theta^{e+1}$ to the sampled weighted average of the uploaded client models, yielding the final objective (Figure~\ref{fig:adaptive-aggregation} (b)):
\begin{equation}
    \arg\min_{\theta^{e+1}} \sum_{i \in S^e} \tilde{w}_i^e \left\| \frac{(\theta^{e+1} - \theta^e)^\top g_i^e}{\|g_i^e\|_2^2 + \varepsilon_{\text{AG}}} - 1 \right\|_2^2 + \alpha_{\text{AG}} \left\| \theta^{e+1} - \left( \theta^e + \sum_{i \in S^e} \tilde{w}_i^e g_i^e \right) \right\|_2^2,
\label{eq:adaptive-aggregation}
\end{equation}
where $\alpha_{\text{AG}}$ balances projection alignment and update regularization.
Note that in the convergence analysis presented in the appendix, these server-side weights are defined as unbiased importance weights under partial client participation.

\section{Algorithm, Complexity, and Privacy}

Due to space limitations, we present key results here, with complete content deferred to the Appendix.

\textbf{Algorithm}: Full algorithm details are provided in the Algorithm~\ref{alg:forgevla-algorithm} in the Appendix.

\textbf{Complexity Analysis}: ForgeVLA matches vanilla FedAvg's asymptotic complexity: $\mathcal{O}(NEP)$ time complexity and $\mathcal{O}(N)$ space complexity, where $N$ is the number of clients, $E$ is the number of communication rounds, and $P$ is the number of local optimization iterations per round.
Please refer to the Appendix for full analysis.

\textbf{Privacy}: ForgeVLA provides standard FL privacy guarantees under the honest-but-curious server model. 
Existing privacy-enhancing paradigms originally designed for general FL are also fully compatible with ForgeVLA.
Please refer to the Appendix for full analysis.

\section{Experiment}

\textbf{Datasets and Models}:
We conduct experiments on four LIBERO benchmarks (LIBERO-Goal, LIBERO-Object, LIBERO-Spatial, LIBERO-10)~\cite{liu2023libero}, which are widely used VLA evaluation suites~\cite{kim2024openvla, qu2025spatialvla, reuss2024mdt, chen2025internvla}. 
For training data, we simulate non-i.i.d. heterogeneity via task-level partitioning, assigning $s$ tasks per client following community practices~\cite{li2020fedprox, yi2024federated, li2024global}.
All local client data retain only vision--action pairs, with a fraction $p$ of the full annotated dataset reserved to fine-tune the embodied instruction classifier, ensuring heterogeneous client data distributions. 
We use the InternVLA-M1 as our backbone~\cite{chen2025internvla}, which follows the encoder-decoder VLA architecture.
For training, we apply LoRA~\cite{hu2022lora} (rank $r$) to the VLM encoder that holds most parameters, and full fine-tuning to the action decoder.

\textbf{Baselines}:
We compare ForgeVLA against a set of representative baselines, including Centralized, FedAvg~\citep{mcmahan2017fedavg}, FedProx~\citep{li2020fedprox}, FedVLA$^*$\footnote{Since official code for FedVLA~\cite{cui2025fedvla} is unavailable, our FedVLA baseline uses a community implementation. Link is omitted for anonymity.}, and a FedVLA$^*$ + CLIP~\citep{radford2021learning} variant.
Specifically, Centralized is a centralized oracle upper bound trained on full aggregated data, and FedVLA$^*$ + CLIP is enhanced with pre-trained CLIP features for better vision-language alignment.
We note that FedVLA$^*$ uses a substantially smaller backbone (79M parameters vs.\ 3.9B for InternVLA-M1), which partly explains its lower absolute performance.

\textbf{Metrics}:
We report four key evaluation metrics: (1) the number of trainable parameters, (2) the total number of model parameters, (3) task success rate, and (4) $\text{Pass}@K$, defined as the probability that a task succeeds at least once within $K$ attempts.

\textbf{Implementation Details}:
We run all experiments with $N=10$ clients with full participation ratio of 1.0 and $s=3$ tasks per client.
Moreover, we use $E=20$ communication rounds and set $P=5$ local optimization iterations per round (implemented as local epochs) to avoid overfitting, with a per-client local batch size of 32.
Additionally, a constant learning rate schedule with warmup is adopted for the VLM encoder ($10^{-5}$) and the action decoder ($10^{-4}$). 
Training uses the AdamW optimizer with zero weight decay, with hyperparameters $\alpha_{\text{CP}}=0.2$, $\alpha_{\text{AG}}=0.1$, and LoRA rank $r=32$ for the VLM encoder.
All experiments are conducted using Nvidia H800 GPUs with CUDA 12.8, Python 3.10.19, and PyTorch 2.6.0. 
For other baselines, we follow the default hyperparameters from their original works or corresponding implementations.
For more details on hyperparameters and training settings, please refer to the Appendix.
Unless otherwise specified, all experiments use the mentioned settings.

\section{Analysis}

\begin{table}[t]
  \def\arraystretch{1.3}
  \addtolength{\tabcolsep}{0pt}
  \centering
  \tablestyle{3pt}{1.0}
  \caption{The performance comparisons between different methods. Default settings are marked in \colorbox{baselinecolor}{gray}. \textbf{bold} marks the best-performing results.}
  \resizebox{\linewidth}{!}{%
    \begin{tabular}{cccccc}
    \toprule
    Methods & Datasets & \# Params (M) & \# Trainable Params (M) & Success Rate (\%) & $\text{Pass}@50$ (\%) \\
    \midrule
    Centralized & LIBERO-Goal & 3882.72 & 128.10 & 75.8  & 100 \\
    FedAvg & LIBERO-Goal & 3882.72 & 128.10 & 28.8  & 80 \\
    FedVLA$^*$ & LIBERO-Goal & 79.23 & 79.23 & 0.2   & 10 \\
    FedVLA$^*$ + CLIP & LIBERO-Goal & 519.82 & 92.90  & 6.2   & 20 \\
    \baseline{ForgeVLA} & \baseline{LIBERO-Goal} & \baseline{3882.72} & \baseline{128.10} & \baseline{\textbf{55.2}\textcolor{ForestGreen}{\scriptsize $^\uparrow$26.4\%}}  & \baseline{100} \\
    \midrule
    Centralized & LIBERO-Object & 3882.72 & 128.10 & 98.8  & 100 \\
    FedAvg & LIBERO-Object & 3882.72 & 128.10 & 97.6  & 100 \\
    FedVLA$^*$ & LIBERO-Object & 79.23 & 79.23 & 2.2   & 10 \\
    FedVLA$^*$ + CLIP & LIBERO-Object & 519.82 & 92.90  & 18.2  & 30 \\
    \baseline{ForgeVLA} & \baseline{LIBERO-Object} & \baseline{3882.72} & \baseline{128.10} & \baseline{\textbf{98.6}\textcolor{ForestGreen}{\scriptsize $^\uparrow$1.0\%}}  & \baseline{100} \\
    \midrule
    Centralized & LIBERO-Spatial & 3882.72 & 128.10 & 85.8  & 100 \\
    FedAvg & LIBERO-Spatial & 3882.72 & 128.10 & 68.6  & 90 \\
    FedVLA$^*$ & LIBERO-Spatial & 79.23 & 79.23 & 0.4   & 10 \\
    FedVLA$^*$ + CLIP & LIBERO-Spatial & 519.82 & 92.90  & 11.6  & 20 \\
    \baseline{ForgeVLA} & \baseline{LIBERO-Spatial} & \baseline{3882.72} & \baseline{128.10} & \baseline{\textbf{72.6}\textcolor{ForestGreen}{\scriptsize $^\uparrow$4.0\%}}  & \baseline{100} \\
    \midrule
    Centralized & LIBERO-10 & 3882.72 & 128.10 & 79    & 100 \\
    FedAvg & LIBERO-10 & 3882.72 & 128.10 & 52.8  & 100 \\
    FedVLA$^*$ & LIBERO-10 & 79.23 & 79.23 & 0.4   & 10 \\
    FedVLA$^*$ + CLIP & LIBERO-10 & 519.82 & 92.90  & 9.4   & 20 \\
    \baseline{ForgeVLA} & \baseline{LIBERO-10} & \baseline{3882.72} & \baseline{128.10} & \baseline{\textbf{63.6}\textcolor{ForestGreen}{\scriptsize $^\uparrow$10.8\%}}  & \baseline{100} \\
    \bottomrule
    \end{tabular}%
    }
  \label{tab:main-results}%
\vspace{-10pt}
\end{table}%

\begin{wraptable}{l}[0pt]{0.4\linewidth}
\vspace{-10pt}
  \def\arraystretch{1.3}
  \addtolength{\tabcolsep}{0pt}
  \centering
  \tablestyle{3pt}{1.0}
  \caption{Well-established non-i.i.d. mitigation strategies such as FedProx~\cite{li2020fedprox} result in even worse performance than vanilla FedAvg.}
    \begin{tabular}{ccc}
    \toprule
    Methods & $\lambda$ & Success Rate (\%) \\
    \midrule
    \multicolumn{3}{c}{Libero-Goal} \\
    \midrule
    FedAvg & 0.0     & 28.0 \\
    FedProx & 0.1   & 11.2 \\
    FedProx & 0.2   & 8.4 \\
    FedProx & 0.5   & 4.2 \\
    \baseline{ForgeVLA} & \baseline{0.0}     & \baseline{\textbf{55.2}} \\
    \bottomrule
    \end{tabular}%
  \label{tab:fedprox-exp}%
\vspace{-10pt}
\end{wraptable}

\textbf{Main Results}
Table~\ref{tab:main-results} presents the main performance comparisons.
Our ForgeVLA achieves the best performance among all federated methods across all settings, significantly outperforming other federated baselines while using the same parameter budget as FedAvg, and substantially narrowing the gap with the centralized training upper bound.
Specifically, for the core success rate metric, ForgeVLA achieves 55.2\% on LIBERO-Goal, a 26.4pp improvement over FedAvg, while existing federated VLA methods nearly fail with success rates below 10\% (noting the backbone capacity difference discussed above).
On LIBERO-Object, ForgeVLA reaches 98.6\%, nearly matching the 98.8\% centralized upper bound.
For $\text{Pass}@50$, ForgeVLA achieves a perfect 100\% across all datasets, on par with the centralized oracle and far exceeding all other federated baselines.
Notably, ForgeVLA introduces no additional trainable parameters compared to FedAvg, further highlighting its efficiency.

\textbf{FedProx Performance Degradation}:
We conduct experiments on LIBERO-Goal to evaluate FedProx~\cite{li2020fedprox} with varying proximal regularization intensity controlled by $\lambda$ in Table~\ref{tab:fedprox-exp}. 
From the table, we observe that FedProx exhibits severe performance degradation as $\lambda$ increases.
Specifically, starting from 28.0\% for vanilla FedAvg ($\lambda=0.0$), the success rate drops to 11.2\% at $\lambda=0.1$, 8.4\% at $\lambda=0.2$, and only 4.2\% at $\lambda=0.5$, while our ForgeVLA framework achieves 55.2\% under the same setting. 
This result directly supports our claim that classical client-drift mitigation strategies designed for standard non-i.i.d. FL are insufficient for heterogeneous VLA federated learning, because the failure mode is vision-language feature collapse (Section~3.2) rather than client drift, and validates the effectiveness of ForgeVLA's targeted approach.

\begin{wraptable}{r}[0pt]{0.53\linewidth}
\vspace{-10pt}
  \def\arraystretch{1.3}
  \addtolength{\tabcolsep}{0pt}
  \centering
  \tablestyle{3pt}{1.0}
  \caption{Data \& model scales for the classifier.}
    \begin{tabular}{cccc}
    \toprule
    $r$     & $p$     & \# Trainable Params (M) & Accuracy (\%) \\
    \midrule
    \multicolumn{4}{c}{Libero-Goal} \\
    \midrule
    4     & 0.9   & 3.36  & 98.34 \\
    4     & 0.5   & 3.36  & 97.96 \\
    4     & 0.2   & 3.36  & 92.30 \\
    4     & 0.1   & 3.36  & 83.56 \\
    4     & ODPT  & 3.36  & 76.68 \\
    8     & ODPT  & 4.61  & 83.34 \\
    16    & ODPT  & 7.10   & 85.56 \\
    32    & ODPT  & 12.10  & 85.67 \\
    \bottomrule
    \end{tabular}%
  \label{tab:scaling-analysis-classifier}%
\vspace{-10pt}
\end{wraptable}

\textbf{Data \& Model Scale Analysis for the Classifier}:
\label{sec:scaling-analysis-classifier}
We evaluate our instruction classifier across varying data scales and model capacities on LIBERO-Goal, with results shown in Table~\ref{tab:scaling-analysis-classifier}. 
Fixed at LoRA rank $r=4$, the classifier retains over 92\% accuracy with 20\% of full training data ($p=0.2$), and maintains 76.68\% accuracy in the most challenging one-data-per-task (ODPT) setting. 
In the ODPT setting, increasing LoRA rank $r$ consistently improves accuracy, with performance saturating at $r=16$ (85.56\%). 
These results confirm our classifier’s high reliability under limited data and small model scales, providing robust support for ForgeVLA.

\begin{wraptable}{r}[0pt]{0.53\linewidth}
\vspace{-10pt}
  \def\arraystretch{1.3}
  \addtolength{\tabcolsep}{0pt}
  \centering
  \tablestyle{3pt}{1.0}
  \caption{Ablation Study of Components in ForgeVLA}
    \begin{tabular}{ccc}
    \toprule
    CP Loss & Adaptive Aggregation & Success Rate (\%) \\
    \midrule
           &       & 28.8 \\
     $\checkmark$     &       & 36.2 \\
           & $\checkmark$     & 49 \\
     $\checkmark$     & $\checkmark$     & 55.2 \\
    \bottomrule
    \end{tabular}%
  \label{tab:ablations}%
\vspace{-5pt}
\end{wraptable}

\textbf{Ablation Study}:
We conduct an ablation study on LIBERO-Goal to validate the contribution of our two core components with $E=10$ and $P=10$ in Table~\ref{tab:ablations}. 
As it can be seen, removing both components recovers the vanilla FedAvg baseline (28.8\%). 
Adding only the CP loss improves performance to 36.2\%, demonstrating its effectiveness in preserving task-specific vision-language alignment. 
Additionally, adding only the adaptive aggregation yields a larger improvement to 49.0\%.
Combining both components achieves the full performance of ForgeVLA at 55.2\%, showing their complementary benefits.

\begin{table*}[hbpt]
\vspace{-5pt}
\centering
\caption{
We conduct an ablation study on ForgeVLA's key hyperparameters with $E=10$ and $P=10$ to evaluate the effectiveness of our design choices. 
Default settings are marked in \colorbox{baselinecolor}{gray}.
\textbf{bold} marks the best-performing results.
}
\vspace{-5pt}
\subfloat[The number of task per client.\label{tab:abl-s}]{
\centering
\begin{minipage}[h]{0.40\textwidth}
\begin{center}

\centering
\tablestyle{5.5pt}{0.9}
\begin{tabular}{ccc}
\toprule
Method & $s$     & Success Rate (\%) \\
\midrule
FedAvg & 1     & 0.8 \\
FedAvg & 2     & 4.2 \\
\baseline{FedAvg} & \baseline{3}     & \baseline{17.6} \\
FedAvg & 5     & 75.6 \\
FedAvg & 10    & 93.2 \\
\midrule
ForgeVLA & 1     & 2 \\
ForgeVLA & 2     & 28.7 \\
\baseline{ForgeVLA} & \baseline{3}     & \baseline{41.6} \\
ForgeVLA & 5     & 85.6 \\
ForgeVLA & 10    & 95.4 \\
\bottomrule
\end{tabular}%

\end{center}
\end{minipage}
}
\hfill
\subfloat[The number of rank during training.\label{tab:abl-r}]{
\centering
\begin{minipage}[h]{0.58\textwidth}
\begin{center}

\centering
\tablestyle{5.5pt}{0.9}
\begin{tabular}{cccc}
\toprule
Method & $r$     & \# Trainable (M) & Success Rate (\%) \\
\midrule
FedAvg & 4     & 119.33 & 9.8 \\
FedAvg & 8     & 120.58 & 14.6 \\
FedAvg & 16    & 123.09 & 18.6 \\
\baseline{FedAvg} & \baseline{32}    & \baseline{128.10} & \baseline{17.6} \\
FedAvg & 64    & 138.13 & 32 \\
\midrule
ForgeVLA & 4     & 119.33 & 36.4 \\
ForgeVLA & 8     & 120.58 & 36.8 \\
ForgeVLA & 16    & 123.09 & 37.2 \\
\baseline{ForgeVLA} & \baseline{32}    & \baseline{128.10} & \baseline{41.6} \\
ForgeVLA & 64    & 138.13 & 42.6 \\
\bottomrule
\end{tabular}%

\end{center}
\end{minipage}}
\\
\vspace{5pt}
\centering
\subfloat[The hyperparameter $\alpha_\text{CP}$.\label{tab:abl-cp}]{
\centering
\begin{minipage}[h]{0.48\textwidth}
\begin{center}

\centering
\tablestyle{10pt}{0.9}
\begin{tabular}{ccc}
\toprule
Method & $\alpha_\text{CP}$ & Success Rate (\%) \\
\midrule
ForgeVLA & 0     & 35.2 \\
ForgeVLA & 0.1   & 36 \\
\baseline{ForgeVLA} & \baseline{0.2}   & \baseline{\textbf{41.6}} \\
ForgeVLA & 0.5   & 37 \\
ForgeVLA & 1     & 36 \\
ForgeVLA & 2     & 33.8 \\
\bottomrule
\end{tabular}%

\end{center}
\end{minipage}
}
\hfill
\subfloat[The hyperparameter $\alpha_\text{AG}$.\label{tab:abl-ag}]{
\centering
\begin{minipage}[h]{0.48\textwidth}
\begin{center}

\centering
\tablestyle{10pt}{1.25}
\begin{tabular}{ccc}
\toprule
Method & $\alpha_\text{AG}$ & Success Rate (\%) \\
\midrule
ForgeVLA & 0     & 37.2 \\
ForgeVLA & 0.05  & 38.8 \\
\baseline{ForgeVLA} & \baseline{0.1}   & \baseline{\textbf{41.6}} \\
ForgeVLA & 0.2   & 35.8 \\
\bottomrule
\end{tabular}%

\end{center}
\end{minipage}
}
\label{tab:sensitivity}
\vspace{-5pt}
\end{table*}

\textbf{Sensitivity Analysis}:
We conduct a sensitivity analysis on four key hyperparameters of ForgeVLA, with results shown in Table~\ref{tab:sensitivity}. 
First, for the number of tasks per client $s$, both FedAvg and ForgeVLA improve as $s$ increases due to richer and more diverse local data, but ForgeVLA still significantly outperforms FedAvg across all $s$, especially in highly heterogeneous settings (\textit{e.g.}, 2\% vs 0.8\% at $s=1$, 41.6\% vs 17.6\% at $s=3$). 
Second, for LoRA rank $r$, ForgeVLA’s performance saturates at the default $r=32$ with only mild gains from larger $r$, and it outperforms FedAvg at all ranks as well, showing no reliance on excessive trainable parameters. 
Third, for $\alpha_\text{CP}$, performance first rises then falls, peaking at the default 0.2, indicating moderate contrastive planning loss preserves task specificity. 
Fourth, for $\alpha_\text{AG}$, performance follows a similar trend, peaking at the default 0.1.
Overall, ForgeVLA remains stable across reasonable hyperparameter ranges, with default settings achieving the best efficiency-performance trade-off.

\begin{figure}[tb]
    \centering
    \begin{minipage}{0.49\linewidth}
        \centering
        \begin{subfigure}[b]{0.23\linewidth}
            \centering
            \includegraphics[width=\linewidth]{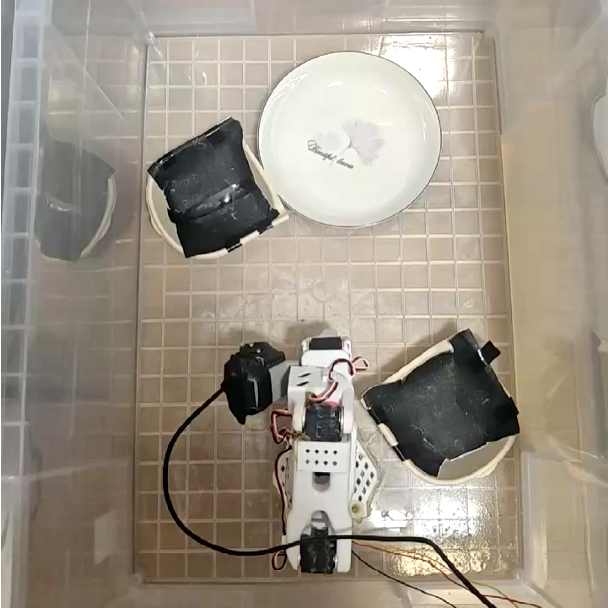}
        \end{subfigure}
        \hfill
        \begin{subfigure}[b]{0.23\linewidth}
            \centering
            \includegraphics[width=\linewidth]{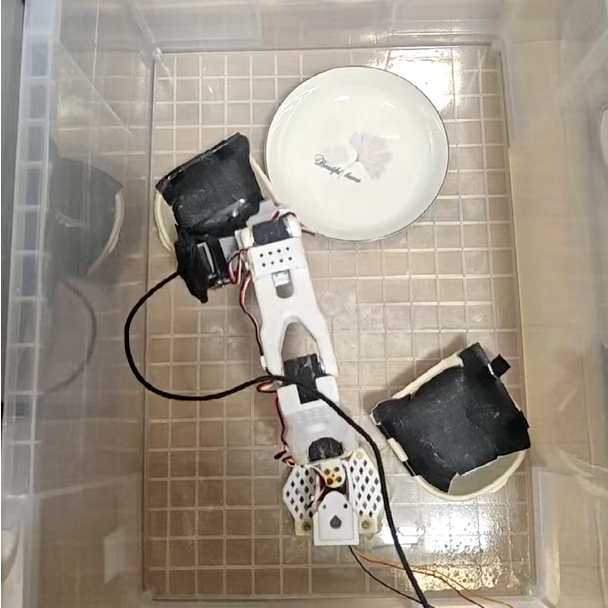}
        \end{subfigure}
        \hfill
        \begin{subfigure}[b]{0.23\linewidth}
            \centering
            \includegraphics[width=\linewidth]{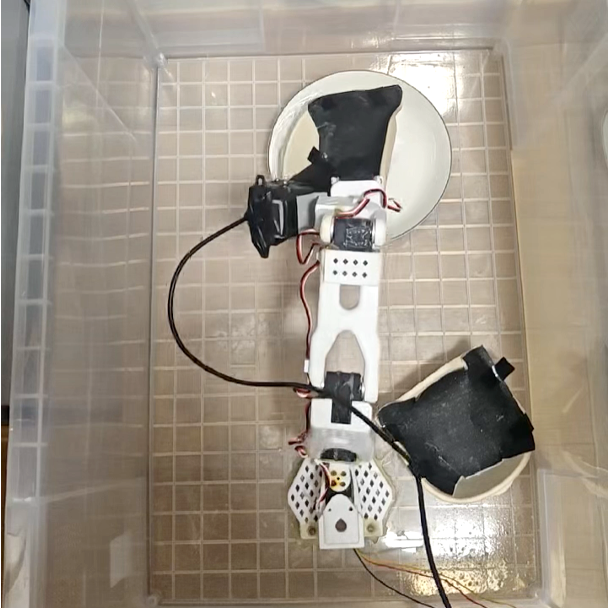}
        \end{subfigure}
        \hfill
        \begin{subfigure}[b]{0.23\linewidth}
            \centering
            \includegraphics[width=\linewidth]{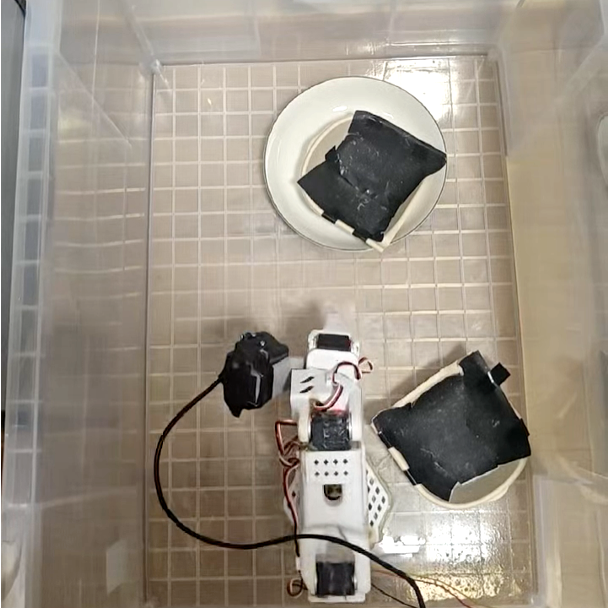}
        \end{subfigure}
        \par
        \vspace{-2pt}
        \centering$\xrightarrow{\hspace{0.8\linewidth}}$
        \vspace{2pt}
        \par
        \begin{subfigure}[b]{0.23\linewidth}
            \centering
            \includegraphics[width=\linewidth]{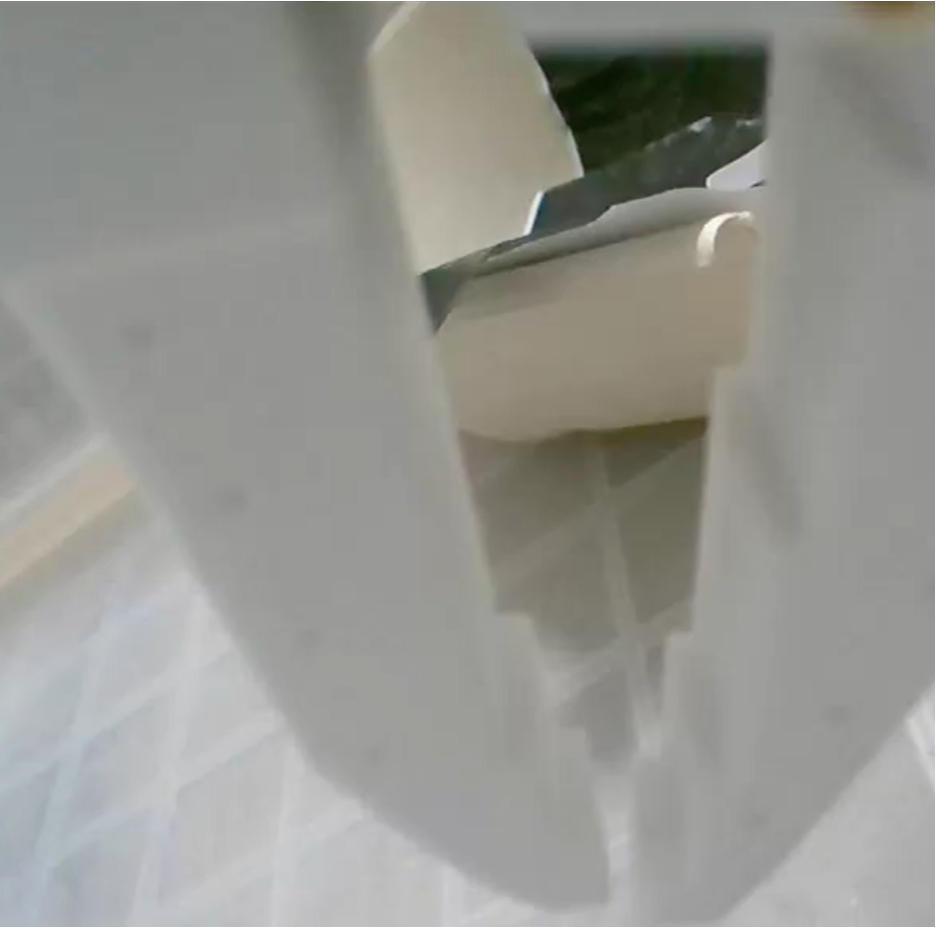}
        \end{subfigure}
        \hfill
        \begin{subfigure}[b]{0.23\linewidth}
            \centering
            \includegraphics[width=\linewidth]{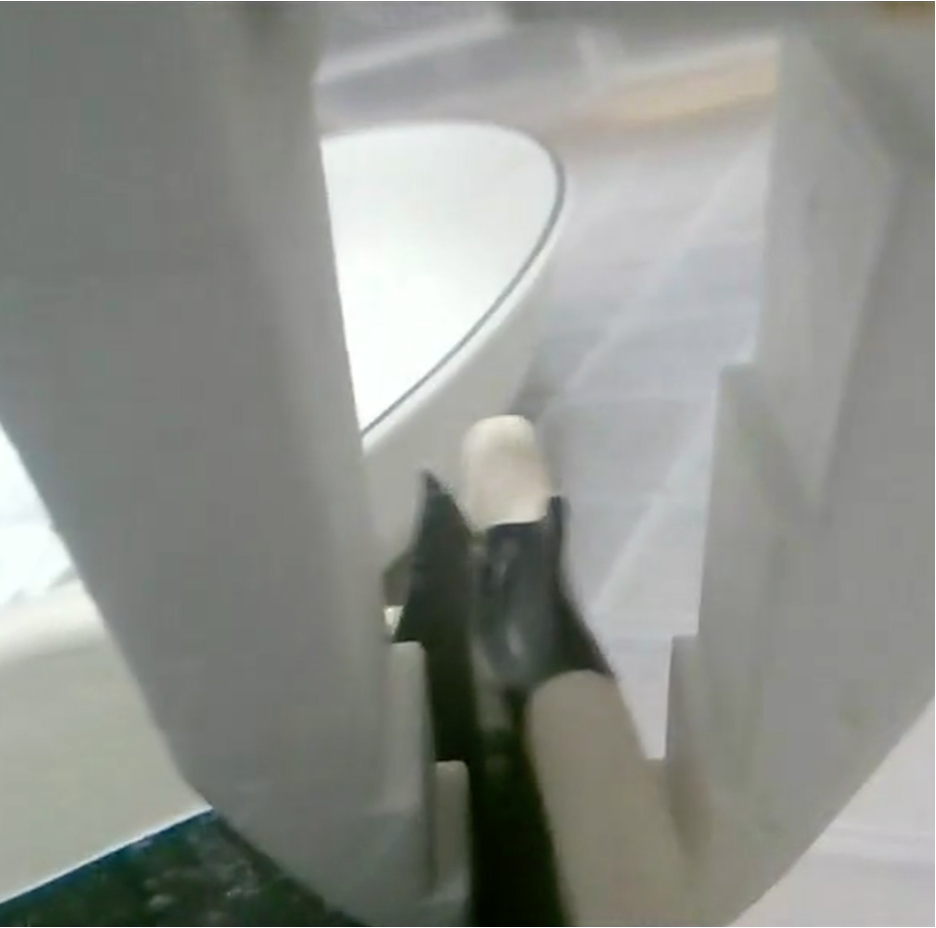}
        \end{subfigure}
        \hfill
        \begin{subfigure}[b]{0.23\linewidth}
            \centering
            \includegraphics[width=\linewidth]{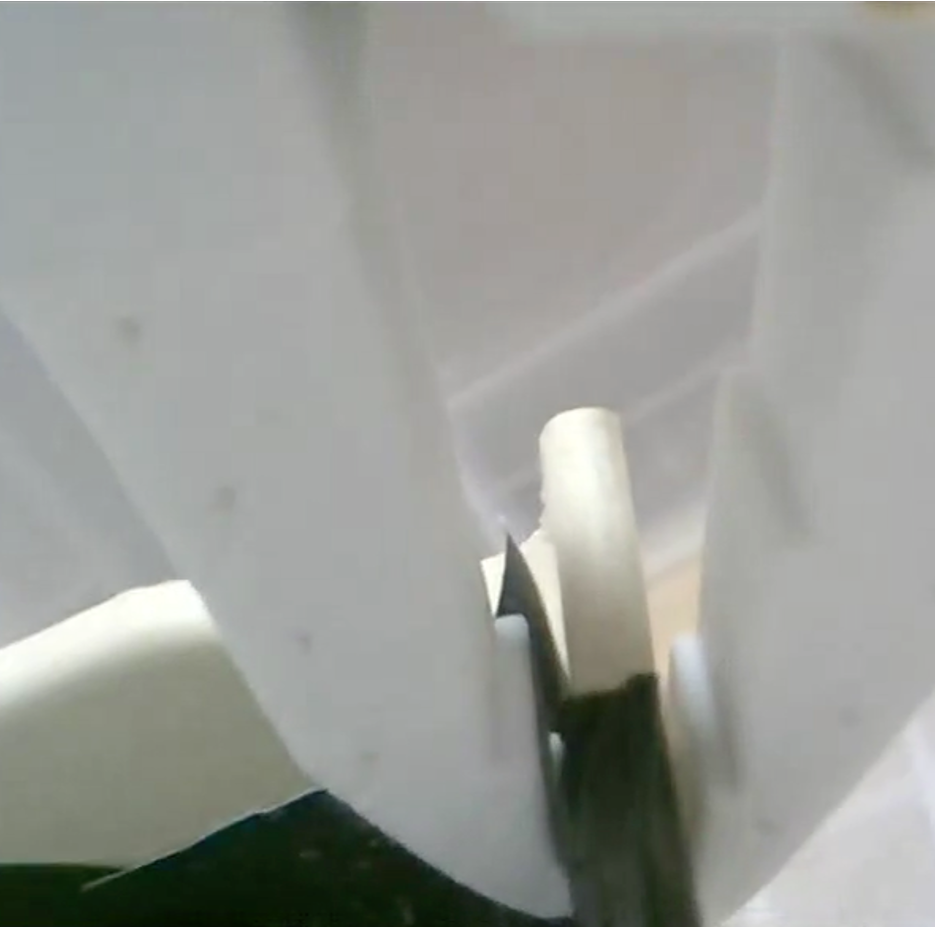}
        \end{subfigure}
        \hfill
        \begin{subfigure}[b]{0.23\linewidth}
            \centering
            \includegraphics[width=\linewidth]{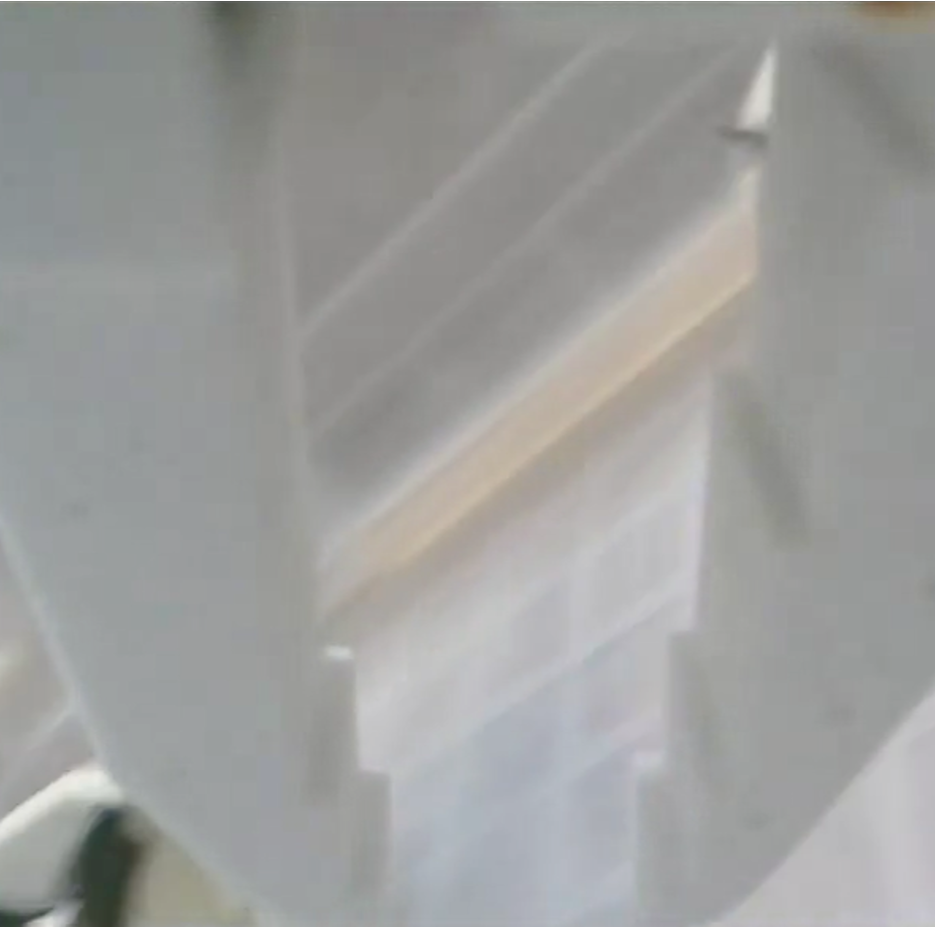}
        \end{subfigure}
        \subcaption{Task: Pick up the black bowl next to the plate and place it on the plate.}
    \end{minipage}
    \hfill
    \begin{minipage}{0.49\linewidth}
        \centering
        \begin{subfigure}[b]{0.23\linewidth}
            \centering
            \includegraphics[width=\linewidth]{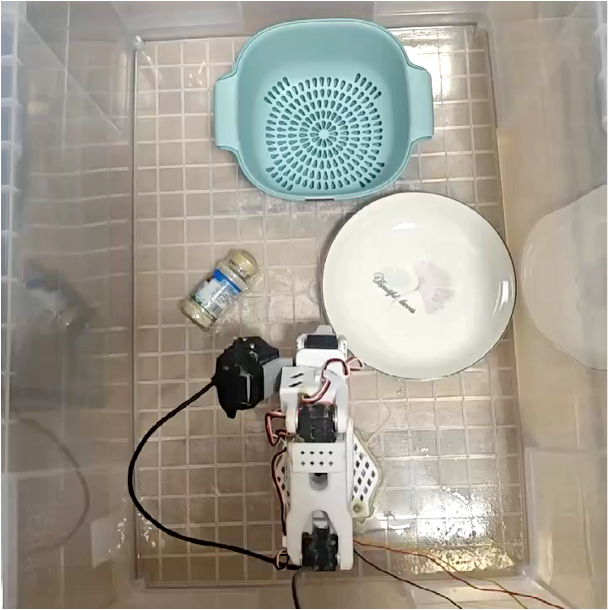}
        \end{subfigure}
        \hfill
        \begin{subfigure}[b]{0.23\linewidth}
            \centering
            \includegraphics[width=\linewidth]{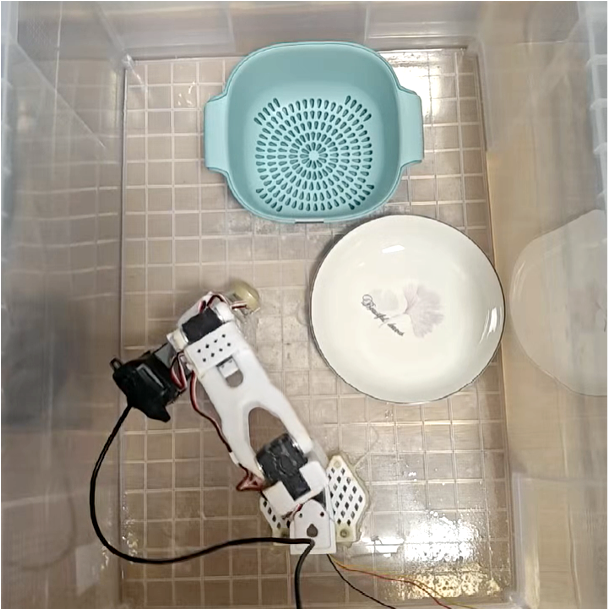}
        \end{subfigure}
        \hfill
        \begin{subfigure}[b]{0.23\linewidth}
            \centering
            \includegraphics[width=\linewidth]{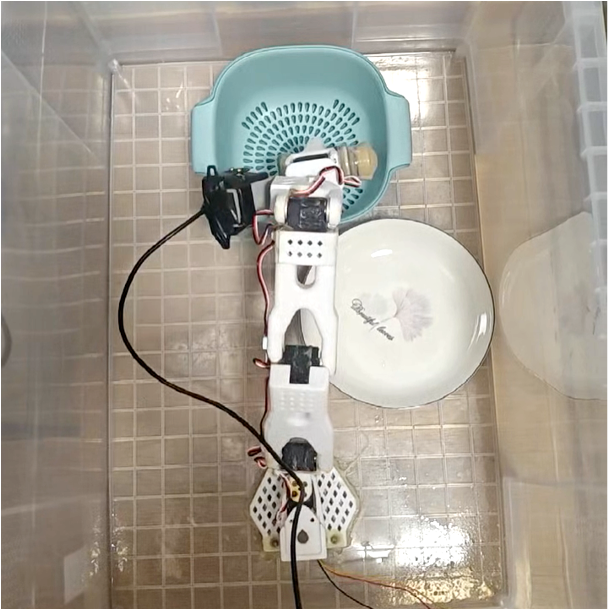}
        \end{subfigure}
        \hfill
        \begin{subfigure}[b]{0.23\linewidth}
            \centering
            \includegraphics[width=\linewidth]{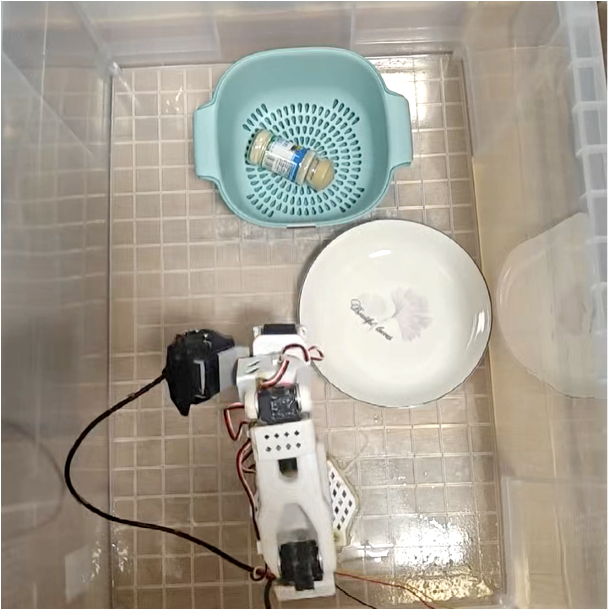}
        \end{subfigure}
        \par
        \vspace{-2pt}
        \centering$\xrightarrow{\hspace{0.8\linewidth}}$
        \vspace{2pt}
        \par
        \begin{subfigure}[b]{0.23\linewidth}
            \centering
            \includegraphics[width=\linewidth]{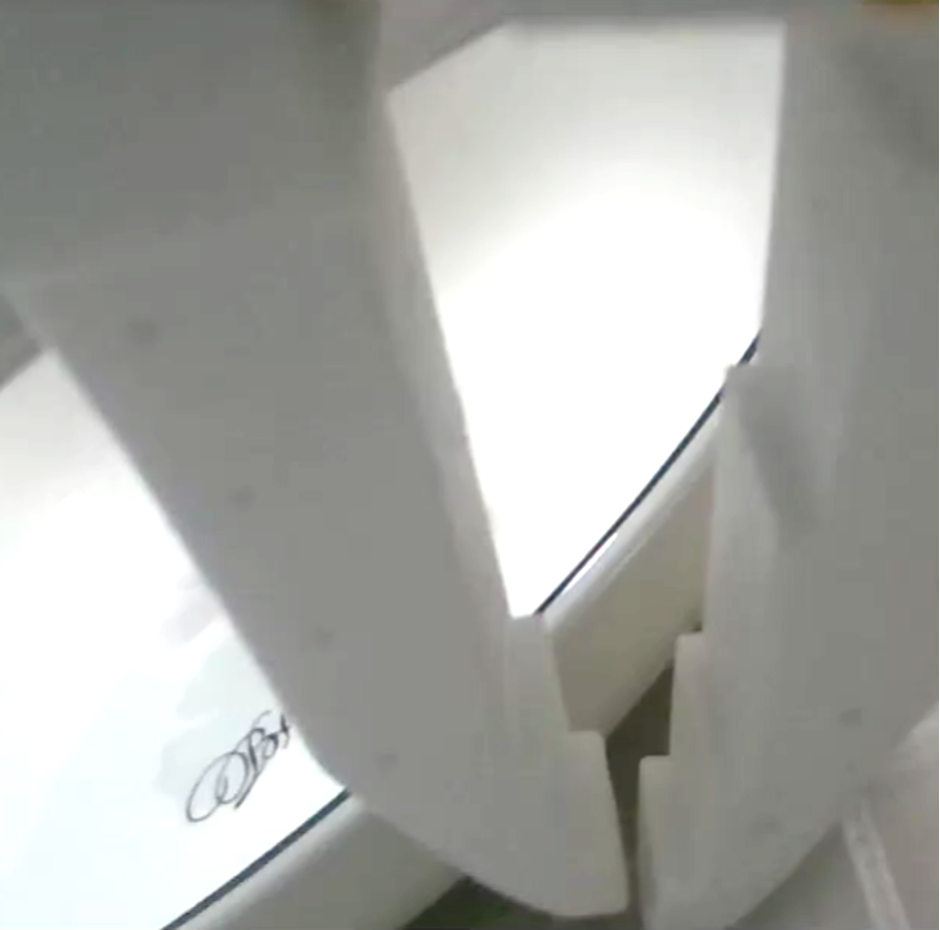}
        \end{subfigure}
        \hfill
        \begin{subfigure}[b]{0.23\linewidth}
            \centering
            \includegraphics[width=\linewidth]{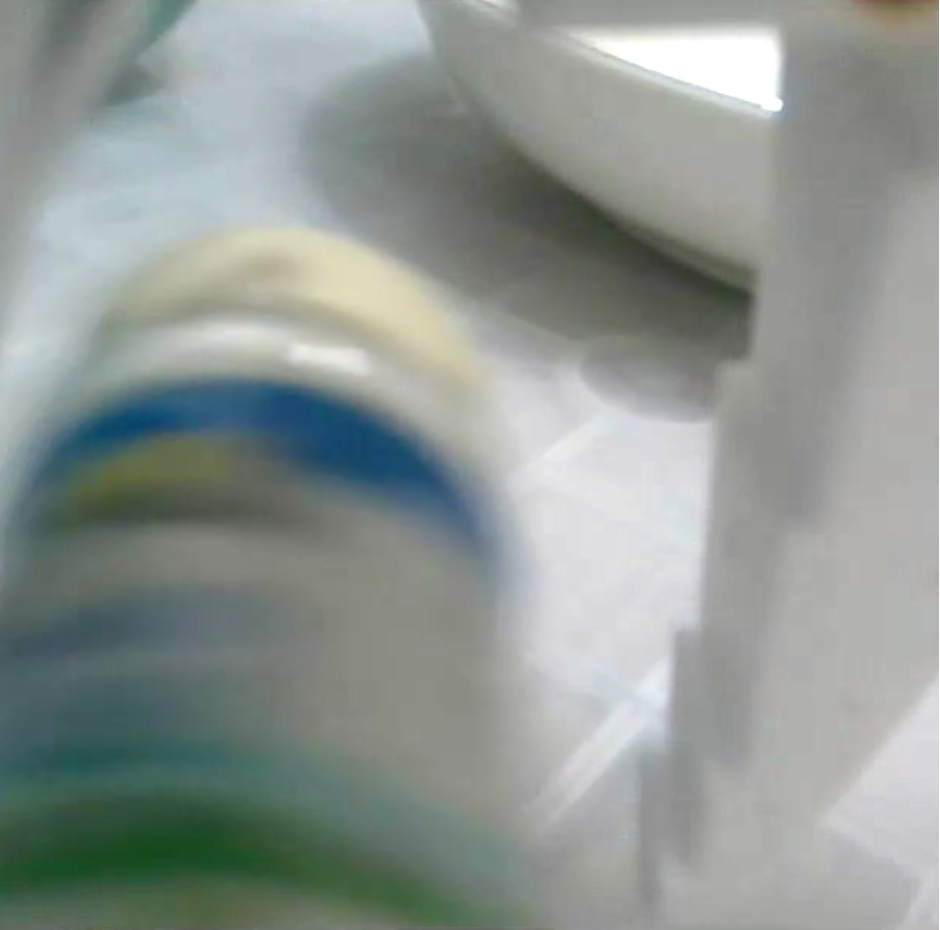}
        \end{subfigure}
        \hfill
        \begin{subfigure}[b]{0.23\linewidth}
            \centering
            \includegraphics[width=\linewidth]{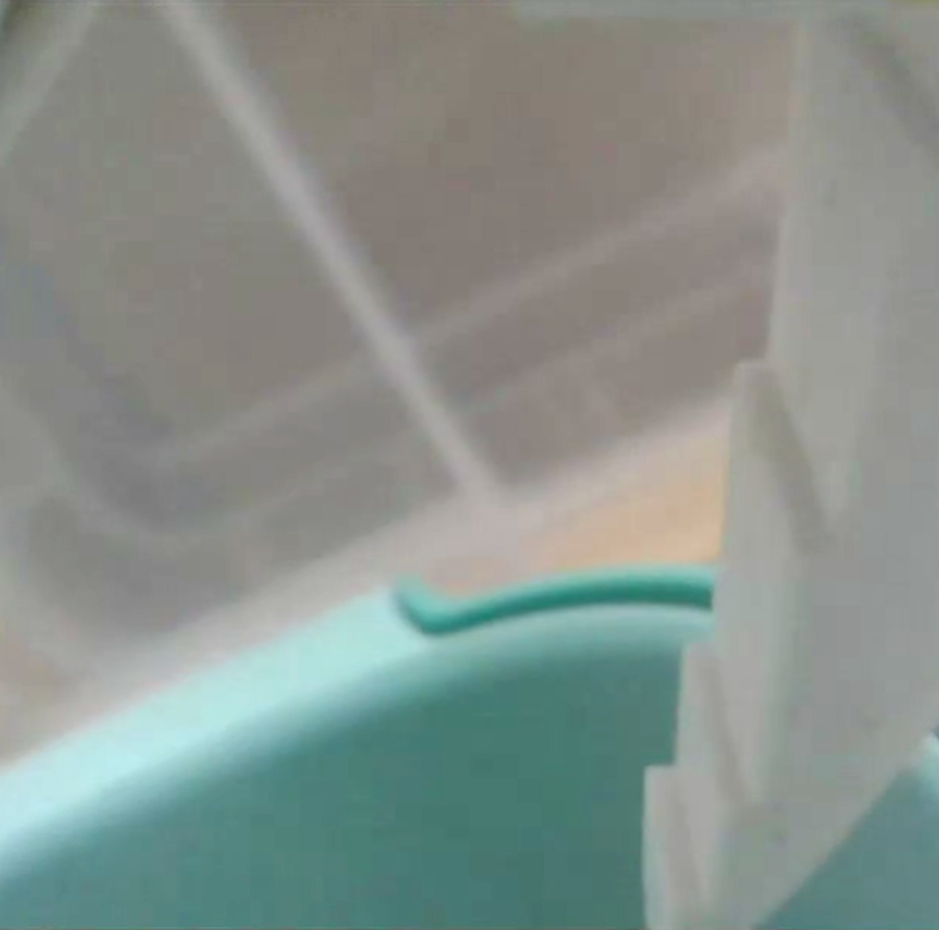}
        \end{subfigure}
        \hfill
        \begin{subfigure}[b]{0.23\linewidth}
            \centering
            \includegraphics[width=\linewidth]{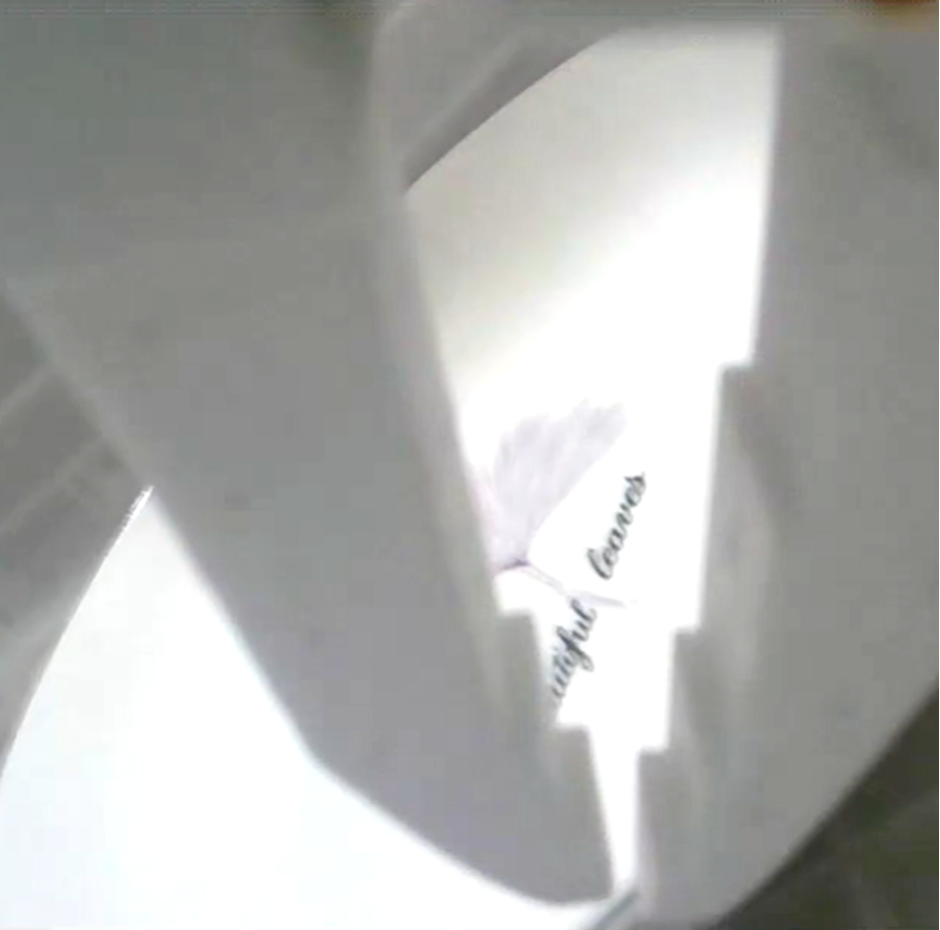}
        \end{subfigure}
        \subcaption{Task: Pick up the barbecue sauce and place it in the basket.}
    \end{minipage}

    \caption{Real-world deployment of the trained ForgeVLA on the SO-ARM101 robotic platform.}
    \label{fig:real-world-deployment}
    \vspace{-15pt}
\end{figure}

\textbf{Real-world Deployment}:
We deploy the trained ForgeVLA model on the physical robotic platform SO-ARM101 and evaluate it on representative tasks to provide a qualitative assessment of real-world transferability, as shown in Fig.~\ref{fig:real-world-deployment}.
The robot successfully executes tested tasks, providing preliminary evidence that the learned policy transfers to physical systems.
Owing to limited resources, this deployment is a proof-of-concept rather than a full real-world evaluation, and a systematic and quantitative study across a broader task set remains future work.

\vspace{-5pt}
\section{Conclusion}
\vspace{-5pt}

We propose ForgeVLA, a federated VLA training framework that unlocks distributed unannotated vision--action robotic logs for privacy-preserving VLA scaling.
ForgeVLA recovers missing language annotations on-device via classification over a predefined instruction set, and mitigates the previously overlooked vision-language feature collapse in heterogeneous federated VLA settings through a contrastive planning loss and an adaptive aggregation strategy.
Extensive experiments show ForgeVLA significantly outperforms federated baselines across benchmarks.
We believe ForgeVLA provides a practical paradigm to harness distributed real-world robotic data for advancing VLA models, and we discuss promising directions for extending the framework, including open-vocabulary instruction recovery, stronger privacy guarantees, etc., in the Appendix.


{
\small

}

\newpage
\appendix

\etocsettocdepth.toc{subsection}
\etocsettocstyle{\section*{Table of Contents}}{}
\etocsetstyle{section}{}{}{\noindent\bfseries\etocnumber\hspace{0.5em}\etocname\dotfill\etocpage\par\vspace{0pt}}{}
\etocsetstyle{subsection}{}{}{\noindent\hspace{1.5em}\etocnumber\hspace{0.5em}\etocname\dotfill\etocpage\par\vspace{0pt}}{}
\tableofcontents
\vspace{1em}

\section{Additional Related Work}

\subsection{Vision Language Action Models and Data Scaling}
Vision Language Action models treat robot control as a single prediction problem over images, instructions, and actions. 
RT-2~\citep{brohan2023rt2} showed that large vision language backbones can be adapted to robotic trajectories, while OpenVLA~\cite{kim2024openvla} and OpenVLA-OFT~\cite{kim2025openoft} made this line easier to study at scale and improved practical efficiency. 
Recent systems push the frontier in different ways. $\pi_0$~\cite{black2024pi0}, CogACT~\cite{li2024cogact}, SpatialVLA~\cite{qu2025spatialvla}, InternVLA-M1~\cite{chen2025internvla}, and DexVLA~\cite{wen2025dexvla} improve dexterous control, spatial reasoning, and action modeling. 
Octo~\cite{ghosh2024octo}, CrossFormer~\cite{doshi2024crossformer}, and HPT~\cite{wang2024hpt} study transfer across embodiments and datasets.
A common pattern is that these methods still depend on carefully annotated demonstrations.

That dependence is hard to scale in robotics. 
Prior work has tried to ease the bottleneck with data augmentation, simulation, and generated trajectories~\citep{laskin2020reinforcement, kostrikov2020image, mandlekar2023mimicgen, james2020rlbench, mu2021maniskill, tobin2017domain, du2024learning, black2024zeroshot, ko2023learning, yu2023scaling}. These approaches are useful, but they still rely on curated seed data or they struggle to produce reliable action supervision in the real world. 
Another practical line focuses on parameter efficient adaptation. 
LoRA~\cite{hu2022lora} is especially relevant because it makes large VLA backbones easier to tune and reduces the communication burden in distributed training~\citep{kim2025openoft, chen2025internvla}. 
ForgeVLA is complementary to all of these efforts. Instead of synthesizing more trajectories, it turns existing vision action logs into usable VLA training data.

\subsection{Federated Learning for Robot Policies and VLA}
Federated learning lets multiple clients train a shared model without uploading raw data~\citep{mcmahan2017fedavg, bonawitz2019towards, kairouz2021advances}. 
Its main difficulty is heterogeneity across clients. 
When local data distributions differ, client updates drift apart and the global model becomes harder to optimize~\citep{zhao2018federated, li2020fedprox, karimireddy2020scaffold, wang2020fednova, acar2021feddyn, reddi2021adaptive}. 
This issue is especially sharp in robot learning because tasks, embodiments, and operating environments vary widely from one client to another.

Several studies bring federated training closer to robotics.
FLAME~\cite{bou2025flame} provides a benchmark for federated manipulation and makes the performance gap under task heterogeneity easy to measure. 
FedVLA~\cite{cui2025fedvla} adapts federated training to VLA models with instruction guided scene parsing, expert routing, and expert aware aggregation. 
Related work on cloud robotics, federated imitation learning, vision and language navigation, and offline reinforcement learning also shows that privacy preserving collaboration is feasible for distributed robot systems~\citep{kehoe2015survey, liu2020fedil, zhou2022fedvln, rengarajan2024fedora}. 
Still, most prior work assumes either fully annotated multimodal data or simpler policy classes. 
ForgeVLA studies a more practical setting in which each client only stores vision action logs and must recover the missing language signal locally.

\subsection{Representation Alignment in Federated Training}
Another relevant thread studies representation alignment as a way to handle non-i.i.d. federated training. 
MOON~\cite{li2021model} uses contrastive objectives to keep local representations close to the global model, and FedProto~\cite{tan2022fedproto} exchanges prototypes so that clients can align features around shared anchors. 
Beyond these two foundational methods, subsequent work has expanded the toolkit considerably. 
FedNTD~\cite{lee2022fedntd} draws an analogy to continual learning and preserves global dark knowledge on non-ground-truth classes via self-distillation during local training. 
FedGen~\cite{zhu2021fedgen} trains a lightweight server-side generator to ensemble client knowledge without any proxy data and broadcasts it to regularize local updates. 
FedDecorr~\cite{shi2023feddecorr} identifies that data heterogeneity causes dimensional collapse, where learned features reside in a lower-dimensional subspace, and applies a decorrelation regularizer to encourage full-rank representations. 
Seo et al.~\cite{seo2024relaxed} show that supervised contrastive learning in FL can itself trigger representation collapse and propose a relaxed contrastive loss that penalizes excessively similar within-class pairs. 
Complementary directions include optimal-transport-based prompt cooperation for aligning global and local representations~\citep{li2024global}, matryoshka-style nested representations that handle model heterogeneity across clients~\citep{yi2024federated}, and personalized priors that reactivate information overlooked during standard aggregation~\citep{shi2023prior}. 
Collectively, these methods show that feature level regularization can be more effective than only constraining parameter updates.

Multimodal federated settings introduce additional challenges. 
CreamFL~\cite{yu2023creamfl} tackles heterogeneous modalities and model architectures by combining inter-modal and intra-modal contrastive losses with a global-local representation ensemble, demonstrating that cross-modal alignment is critical when clients observe different data modalities. 
Federated VLA makes this problem even harder because the model must preserve agreement between visual content, language instructions, and action predictions in one shared latent space. 
In our setting, weak alignment leads to vision language feature collapse, where task representations lose separation across clients. Generic contrastive FL methods do not fully address this multimodal failure mode. 
ForgeVLA therefore uses a contrastive planning loss with a shared task representation bank so that inter task structure remains stable throughout training.

\section{Algorithm of ForgeVLA}

The complete pseudocode of ForgeVLA is presented in Algorithm~\ref{alg:forgevla-algorithm}.

\begin{algorithm}[htbp]
    \caption{ForgeVLA}
    \label{alg:forgevla-algorithm}
    \textbf{Input}: initial global model parameters $\theta^0$ and $\phi$, local dataset $D_i=\{(v_i^k, a_i^k)\}_{k=1}^{K_i}$, local learning rate $\eta_i$\\
    \textbf{Parameter}: number of communication rounds $E$, number of local optimization iterations $P$, number of clients $N$, participation rule $\{p_i^e\}$\\
    \textbf{Output}: global model parameters $\theta^E$\\
    \textbf{Server:}
    \begin{algorithmic}[1] %
        \STATE fine-tune $\phi$ on public VLA pairs
        \STATE initialize the normalized global task bank $U^0 = \{ \hat{u}_1^0, \hat{u}_2^0, \dots, \hat{u}_M^0 \}$
        \FOR{each communication round $e$ from $0$ to $E-1$}
            \STATE sample a participating set $S^e \subseteq [N]$ according to $\{p_i^e\}_{i=1}^N$
            \STATE distribute $\theta^e$, $\phi$, and $U^e$ to the clients in $S^e$
            \STATE each client in $S^e$ executes the routine below in parallel
            \STATE receive $\theta_i^{e,P}$ and $\{ u_{i,1}^e, u_{i,2}^e, \dots, u_{i,M}^e \}$ from each client $i \in S^e$
            \STATE update $U^{e+1}$ using weighted aggregation over the uploaded prototypes; if a task receives no valid aggregate, keep its previous bank entry, and renormalize each updated entry
            \STATE update $\theta^{e+1}$ using Equation~\ref{eq:adaptive-aggregation} on the uploaded client set $S^e$
        \ENDFOR
        \RETURN $\theta^E$
    \end{algorithmic}
    \textbf{Client Routine (executed on participating client $i \in S^e$ at round $e$):}
    \begin{algorithmic}[1] %
        \STATE initialize $\theta_i^{e,0} = \theta^e$
        \FOR{each local iteration $j$ from $1$ to $P$}
            \STATE sample a mini-batch $\xi_i^{e,j}$ from $D_i$
            \STATE form an unbiased stochastic gradient $h_i^{e,j}$ of $\mathcal{L}_i + \mathcal{L}_{\text{CP},i}^e$ at $(\theta_i^{e,j-1}, \xi_i^{e,j})$
            \STATE $\theta_i^{e,j} = \theta_i^{e,j-1} - \eta_i h_i^{e,j}$
        \ENDFOR
        \STATE update $\{ u_{i,1}^e, u_{i,2}^e, \dots, u_{i,M}^e \}$ using $\theta_i^{e,P}$
        \RETURN $\theta_i^{e,P}$ and $\{ u_{i,1}^e, u_{i,2}^e, \dots, u_{i,M}^e \}$
    \end{algorithmic}
\end{algorithm}

\section{Complexity Analysis}

In each communication round, each participating client uses the pretrained embodied instruction classifier to convert raw vision–action logs into a VLA training corpus, and then performs local optimization by jointly minimizing the VLA task loss and the contrastive planning loss. 
The server aggregates the uploaded client updates via adaptive aggregation and updates the global task representation bank.
Throughout, it never accesses raw client data.
Therefore, ForgeVLA incurs no additional asymptotic time complexity beyond vanilla FedAvg, and its runtime remains $\mathcal{O}(NEP)$, where $E$ is the number of communication rounds and $P$ is the number of local optimization iterations per round.
In space, the instruction classifier introduces only constant extra storage, and the global task representation bank scales linearly with the number of tasks $M$ while remaining lightweight in practice.
Consequently, ForgeVLA matches FedAvg’s $\mathcal{O}(N)$ space complexity.

\begin{table}[hbpt]
    \centering
    \caption{The glossary of notations.}
    \resizebox{1.0\linewidth}{!}{
    \begin{tabular}{p{0.2\linewidth}p{0.8\linewidth}}
        \hline
        Notation & Implication \\
        $N$ & Total number of clients \\
        $i$ & Index of a client \\
        $k$ & Sample index within the local dataset $D_i$ \\
        $e$ & Communication-round index \\
        $j$ & Local optimization-iteration index within round $e$ \\
        $S^e$ & Participating client set at communication round $e$ \\
        $\xi_i^e$ & Participation indicator of client $i$ at round $e$ \\
        $\xi_i^{e,j}$ & Mini-batch sampled on client $i$ at local iteration $j$ of round $e$ \\
        $w_i$ & Normalized analysis weight of the $i$-th client, \textit{e.g.}, $w_i = |D_i| / \sum_{j=1}^N |D_j|$, satisfying $\sum_{i=1}^N w_i = 1$ \\
        $\tilde{w}_i^e$ & Unbiased sampled server weight for client $i$ at round $e$ \\
        $E$ & Number of communication rounds \\
        $P$ & Number of local optimization iterations per round \\
        $\eta_i$ & Local learning rate for client $i$ \\
        $D_i$ & Local dataset of client $i$ \\
        $K_i$ & Number of samples in the local dataset $D_i$ \\
        $(v_i^k, a_i^k)$ & The $k$-th vision--action pair in $D_i$ \\
        $l_m$ & The $m$-th instruction in the instruction space $\mathcal{L}$ \\
        $\hat{l}_i^k$ & Predicted instruction from the embodied instruction classifier \\
        $t_i^k$ & Instruction index satisfying $\hat{l}_i^k = l_{t_i^k}$ \\
        $s$ & Number of tasks assigned per client \\
        $p$ & Fraction of annotated data reserved for classifier fine-tuning \\
        $\mathcal{V}$ & Vision observation space \\
        $\mathcal{A}$ & Action space \\
        $\mathcal{L} = \{l_m\}_{m=1}^M$ & Instruction space \\
        $\mathcal{Z}$ & Joint latent embedding space \\
        $M$ & Total number of distinct language instructions \\
        $\theta = (\theta_{\text{enc}}, \theta_{\text{dec}})$ & Global VLA model parameters (encoder and decoder) \\
        $\theta^e$ & Global VLA model at the start of round $e$ \\
        $\theta_i^{e,j}$ & Local VLA model of client $i$ after $j$ local iterations in round $e$ \\
        $f_\theta(v, l)$ & Full VLA model mapping from vision-language input to action \\
        $f_{\theta_{\text{enc}}}: \mathcal{V} \times \mathcal{L} \rightarrow \mathcal{Z}$ & Vision-language encoder \\
        $f_{\theta_{\text{dec}}}: \mathcal{Z} \rightarrow \mathcal{A}$ & Action decoder \\
        $z_i^k$ & Joint latent representation of sample $(v_i^k, a_i^k)$ \\
        $\hat{z}_i^k$ & $\ell_2$-normalized latent representation of sample $(v_i^k, a_i^k)$ \\
        $r$ & LoRA rank applied to the VLM encoder \\
        $c_\phi: \mathcal{V} \times \mathcal{A} \rightarrow \mathcal{L}$ & Embodied instruction classifier \\
        $\phi \in \Phi$ & Parameters of the instruction classifier \\
        $\Phi$ & Hypothesis space of the classifier parameters \\
        $\ell(\cdot, \cdot)$ & Task-specific loss function \\
        $\mathcal{L}_i(\theta; D_i, \phi)$ & Local empirical VLA loss on client $i$ \\
        $\mathcal{L}_{\text{CP},i}^e$ & Contrastive planning loss on client $i$ at round $e$ \\
        $\alpha_{\text{CP}}$ & Scaling hyperparameter for the contrastive planning loss \\
        $\tau$ & Fixed temperature hyperparameter in the contrastive planning loss \\
        $U^e = \{\hat{u}_1^e, \hat{u}_2^e, \dots, \hat{u}_M^e\}$ & Normalized global task representation bank broadcast at round $e$ \\
        $u_{i,m}^e$ & Local task representation of instruction $m$ on client $i$ at round $e$ \\
        $g_i^e$ & Local update of client $i$ at round $e$, \textit{i.e.}, $\theta_i^{e,P} - \theta^e$ \\
        $\alpha_{\text{AG}}$ & Hyperparameter balancing projection alignment and regularization \\
        $\varepsilon_{\text{AG}}$ & Numerical stabilizer in adaptive aggregation \\
        $\text{Pass}@K$ & Probability that a task succeeds at least once within $K$ attempts \\
        $\lambda$ & Proximal regularization intensity in FedProx \\
        \hline
    \end{tabular}
    }
    \label{tab:glossary}
\end{table}

\section{Notation}
The main notations in this paper are shown in Table~\ref{tab:glossary}.

\section{Convergence Analysis}
\label{sec:convergence}

We analyze ForgeVLA as local SGD on the round-frozen objectives \(\{F_i^e\}\) induced by the current bank \(U^e\)~\citep{stich2019local, khaled2020local, karimireddy2020scaffold, reddi2021adaptive}. The purpose of this analysis is to provide a \emph{stability-style guarantee}: under bounded algorithm-induced perturbations, adaptive aggregation and the evolving task bank preserve the FedAvg/local-SGD stationarity order. The theorem does not claim that these components explain the empirical advantage of ForgeVLA over baselines. The analysis is stated as one theorem in which stochastic noise, client heterogeneity, partial participation, adaptive aggregation, and bank drift appear in a single bound.

Let \(n_i \triangleq |D_i|\) and define the normalized client weights
\(
w_i \triangleq n_i / \sum_{k=1}^N n_k
\),
so \(w_i \geq 0\) and \(\sum_{i=1}^N w_i = 1\). We assume a common local step size \(\eta > 0\) and \(P \in \mathbb{N}_+\) local steps per round.
The two extra ForgeVLA hyperparameters are coupled to the local-SGD scale as
\begin{equation}
\alpha_{\text{CP}}
=
\lambda_{\text{CP}} \eta,
\qquad
\alpha_{\text{AG}}
=
\frac{\lambda_{\text{AG}}}{\eta P},
\qquad
\lambda_{\text{CP}} \geq 0,
\quad
\lambda_{\text{AG}} > 0.
\label{eq:theory-hyperparameter-scaling}
\end{equation}
The first scaling keeps the bank term at the same order as the usual stochastic floor, and the second serves as a convenient asymptotic parametrization of the server anchor strength as the local-update scale shrinks. The theorem itself relies on Assumption~\ref{assump:forgevla-regularity} to quantify the objective-level perturbation induced by adaptive aggregation, rather than on this heuristic alone.

As in the rest of the paper, the bank entries are row-wise normalized; for analysis we replace the raw normalization \(\hat{z} = z/\|z\|_2\) by the smooth map
\begin{equation}
\hat{z}
\triangleq
\operatorname{norm}_{\varepsilon_{\mathrm{n}}}(z)
\triangleq
\frac{z}{\sqrt{\|z\|_2^2 + \varepsilon_{\mathrm{n}}^2}},
\qquad
\varepsilon_{\mathrm{n}} > 0,
\end{equation}
so that \(\|\hat{z}\|_2 \leq 1\) deterministically.
This \(\varepsilon_{\mathrm{n}}\)-smoothed map is used only as an analytical surrogate for the proof: the implemented algorithm uses standard row-wise normalization, while the smoothed version removes the singularity at \(z = 0\) and approximates the raw normalization arbitrarily well once \(\|z\|_2 \gg \varepsilon_{\mathrm{n}}\).
At communication round \(e\), the server broadcasts a bank
\(
U^e = [u_1^e;\dots;u_M^e]
\)
with \(\|u_m^e\|_2 \leq 1\) for every \(m\), and this bank is held fixed throughout the subsequent \(P\) local updates.
For one sample with target task \(t\), and with the round-\(e\) bank \(U^e\), define
\begin{equation}
\ell_{\text{CP}}(\hat{z}; U^e, t)
\triangleq
- \alpha_{\text{CP}}
\left(
\frac{\hat{z}^{\top}u_t^e}{\tau}
- \log \sum_{m=1}^{M} \exp\left(\frac{\hat{z}^{\top}u_m^e}{\tau}\right)
\right),
\qquad
\tau > 0.
\end{equation}
Since \(\log \sum_{m=1}^{M} \exp(a_m) \geq a_t\), we have
\(
\ell_{\text{CP}}(\hat{z}; U^e, t) \geq 0
\)
for every \((\hat{z}, U^e, t)\).
Throughout this subsection, \(\mathcal{L}_i(\theta)\) denotes the empirical-average VLA loss on \(D_i\), and \(\mathcal{L}_{\text{CP},i}^e(\theta)\) denotes the empirical average of \(\ell_{\text{CP}}(\hat{z}; U^e, t)\) over the samples in \(D_i\).
For brevity, write \(\mathcal{L}_i(\theta) \equiv \mathcal{L}_i(\theta; D_i, \phi)\) and \(\mathcal{L}_{\text{CP},i}^e(\theta) \equiv \mathcal{L}_{\text{CP},i}^e(\theta; D_i, \phi, U^e)\), and define the round-frozen local and global objectives as
\begin{equation}
F_i^e(\theta) \triangleq \mathcal{L}_i(\theta) + \mathcal{L}_{\text{CP},i}^e(\theta),
\qquad
F^e(\theta) \triangleq \sum_{i=1}^N w_i F_i^e(\theta).
\end{equation}
Let \(\mathcal{F}_e\) denote the sigma-field generated by all randomness up to the start of round \(e\), and let \(\mathcal{F}_{e,j}\) denote the sigma-field generated by all randomness up to and including the \(j\)-th local update on every client in the coupled full-participation round-\(e\) recursion, with \(\mathcal{F}_{e,0} = \mathcal{F}_e\).
Then \(\theta^e\) and \(U^e\) are \(\mathcal{F}_e\)-measurable, and conditioned on \(\mathcal{F}_e\) the inner loop is standard local SGD on \(\{F_i^e\}_{i=1}^N\).

Conditioned on the round-start sigma-field \(\mathcal{F}_e\), each client starts from \(\theta_i^{e,0} = \theta^e\) and performs \(P\) local SGD steps on \(F_i^e\):
\begin{equation}
\theta_i^{e,j}
=
\theta_i^{e,j-1} - \eta h_i^{e,j},
\qquad
j=1,\dots,P,
\end{equation}
where \(h_i^{e,j}\) is a stochastic gradient estimator of \(F_i^e\) at \(\theta_i^{e,j-1}\).
After the local stage, define
\begin{equation}
\bar{\theta}^{e+1}
\triangleq
\sum_{i=1}^N w_i \theta_i^{e,P},
\qquad
g_i^e \triangleq \theta_i^{e,P} - \theta^e,
\qquad
\bar{g}^e \triangleq \bar{\theta}^{e+1} - \theta^e.
\end{equation}
To model partial participation, let \(\xi_i^e \in \{0,1\}\) denote the round-\(e\) participation indicator of client \(i\), let \(p_i^e \triangleq \mathbb{P}(\xi_i^e = 1 \mid \mathcal{F}_e) \in (0,1]\), and define the unbiased sampled server weights and sampled FedAvg anchor by
\begin{equation}
\tilde{w}_i^e
\triangleq
\frac{w_i \xi_i^e}{p_i^e},
\qquad
\tilde{\theta}^{e+1}
\triangleq
\theta^e + \sum_{i=1}^N \tilde{w}_i^e g_i^e.
\label{eq:sampled-weights-anchor}
\end{equation}
Equivalently, all sums below may be restricted to the uploaded client set because \(\tilde{w}_i^e = 0\) for non-participating clients.
The server then applies the adaptive-aggregation update
\begin{equation}
\theta^{e+1}
\triangleq
\arg\min_{\theta}
\sum_{i=1}^N \tilde{w}_i^e
\left(
\frac{(\theta-\theta^e)^{\top} g_i^e}{\|g_i^e\|_2^2 + \varepsilon_{\text{AG}}}
- 1
\right)^2
+ \alpha_{\text{AG}}
\left\|
\theta - \tilde{\theta}^{e+1}
\right\|_2^2,
\label{eq:ag-analysis-update}
\end{equation}
where \(\varepsilon_{\text{AG}} > 0\).
Because the objective in Eq.~\eqref{eq:ag-analysis-update} is strongly convex, \(\theta^{e+1}\) is uniquely defined.

To analyze the bank update, let \(R_i(\theta) = [r_{i,1}(\theta);\dots;r_{i,M}(\theta)]\) denote the task-prototype matrix uploaded by client \(i\) when its local model is \(\theta\).
Let \(\mathcal{M}_{\mathrm{act}} \subseteq [M]\) denote this active row set, and write
\(
M_{\mathrm{act}} \triangleq |\mathcal{M}_{\mathrm{act}}|
\).
For each active task \(m \in \mathcal{M}_{\mathrm{act}}\), the server computes
\begin{equation}
\tilde{u}_m^{e+1}
\triangleq
\sum_{i=1}^N \rho_{i,m}^e \, r_{i,m}(\theta_i^{e,P}),
\qquad
u_m^{e+1}
\triangleq
\operatorname{norm}_{\varepsilon_{\mathrm{n}}}(\tilde{u}_m^{e+1}),
\label{eq:exact-bank-refresh}
\end{equation}
where \(\rho_{i,m}^e \geq 0\), \(\sum_{i=1}^N \rho_{i,m}^e = 1\), and \(\rho_{i,m}^e = 0\) whenever client \(i\) does not upload task \(m\) in round \(e\).
For \(m \notin \mathcal{M}_{\mathrm{act}}\), the server keeps \(u_m^{e+1} = u_m^e\).

\begin{assumption}[Round-Wise Smoothness]
\label{assump:smooth}
For every round \(e\) and client \(i\), the objective \(F_i^e\) is \(L\)-smooth, \textit{i.e.}, for all \(\theta,\theta'\),
\begin{equation}
\|\nabla F_i^e(\theta) - \nabla F_i^e(\theta')\|_2 \leq L\|\theta - \theta'\|_2.
\end{equation}
\end{assumption}
The smoothed normalization keeps the contrastive term well defined, and the assumption above is the standard round-wise smoothness condition used in nonconvex local-SGD and federated optimization analyses~\citep{stich2019local, khaled2020local, karimireddy2020scaffold, reddi2021adaptive}.

\begin{assumption}[Uniform Lower Bound on the Round-Wise Objectives]
\label{assump:lower-bound}
There exists a deterministic constant \(F_{\mathrm{lb}} > -\infty\) such that
\begin{equation}
F^e(\theta) \geq F_{\mathrm{lb}}
\qquad
\text{for all rounds \(e\) and all \(\theta\).}
\end{equation}
\end{assumption}
Uniform lower boundedness is the usual stationarity-analysis condition for smooth nonconvex optimization and local SGD~\citep{stich2019local, reddi2021adaptive}.

Because only the contrastive term depends on the bank, the inter-round objective drift is controlled directly by the bank movement.
For the sample loss \(\ell_{\text{CP}}(\hat{z}; U, t)\),
\begin{equation}
\nabla_{u_m} \ell_{\text{CP}}
=
\frac{\alpha_{\text{CP}}}{\tau}\left(p_m - \mathbf{1}\{m=t\}\right)\hat{z},
\end{equation}
where \(p_m\) is the corresponding softmax probability.
Hence
\begin{equation}
\|\nabla_U \ell_{\text{CP}}\|_F^2
=
\frac{\alpha_{\text{CP}}^2}{\tau^2}
\|\hat{z}\|_2^2
\sum_{m=1}^{M}\left(p_m - \mathbf{1}\{m=t\}\right)^2
\leq
\frac{2\alpha_{\text{CP}}^2}{\tau^2},
\end{equation}
because \(\|\hat{z}\|_2 \leq 1\) and \(\|p - e_t\|_2 \leq \sqrt{2}\).
Therefore \(\ell_{\text{CP}}\) is \((\sqrt{2}\alpha_{\text{CP}}/\tau)\)-Lipschitz in \(U\), and taking empirical averages and the client-weighted sum preserves this constant:
\begin{equation}
|F^{e+1}(\theta) - F^e(\theta)|
\leq
\frac{\sqrt{2}\alpha_{\text{CP}}}{\tau}\|U^{e+1} - U^e\|_F
\qquad
\text{for all rounds \(e\) and all \(\theta\).}
\label{eq:objective-variation-bank}
\end{equation}

\begin{lemma}[Bank Variation Implies Objective Drift]
\label{lem:bank-variation}
Suppose there exists \(C_U \geq 0\) such that, for every round \(e\),
\begin{equation}
\mathbb{E}\!\left[
\|U^{e+1} - U^e\|_F
\middle|
\mathcal{F}_e
\right]
\leq
C_U \eta P.
\label{eq:bank-variation}
\end{equation}
Then, under \(\alpha_{\mathrm{CP}} = \lambda_{\mathrm{CP}}\eta\),
\begin{equation}
\mathbb{E}\!\left[
F^{e+1}(\theta^{e+1}) - F^e(\theta^{e+1})
\middle|
\mathcal{F}_e
\right]
\leq
\frac{\sqrt{2}\lambda_{\mathrm{CP}}C_U}{\tau}\eta^2P.
\label{eq:bank-variation-drift}
\end{equation}
\end{lemma}
\begin{proof}
Eq.~\eqref{eq:objective-variation-bank} holds pointwise for any \(\theta\), including the random iterate \(\theta^{e+1}\). Taking conditional expectation and applying Eq.~\eqref{eq:bank-variation} gives Eq.~\eqref{eq:bank-variation-drift}.
\end{proof}
Eq.~\eqref{eq:bank-variation} is a bank path-length condition, not an optimal-bank convergence claim. It is mild for prototype banks because \(U^e\) is an averaged representation statistic rather than an adversarial objective sequence, matching the prototype-sharing view used in federated representation learning~\citep{tan2022fedproto}. If the joint model-bank dynamics approach a stable representation fixed point, then \(\|U^{e+1}-U^e\|_F \to 0\); the rate theorem assumes the quantitative per-round control needed for an \(\mathcal{O}(E^{-1/2})\) stationarity bound. Bounded variation/path-length controls are standard in non-stationary and time-varying optimization~\citep{besbes2015nonstationary, lin2023timevarying}; related surrogate analyses also control changes between successive optimization surrogates~\citep{mairal2013surrogate, razaviyayn2013bsum}.

\begin{assumption}[Conditionally Unbiased Stochastic Gradients with Bounded Noise]
\label{assump:variance}
For every round \(e\), client \(i\), and local step \(j\),
\begin{equation}
\mathbb{E}\!\left[h_i^{e,j}\mid \mathcal{F}_{e,j-1}\right]
=
\nabla F_i^e(\theta_i^{e,j-1}),
\qquad
\mathbb{E}\!\left[
\left\|
h_i^{e,j} - \nabla F_i^e(\theta_i^{e,j-1})
\right\|_2^2
\middle|
\mathcal{F}_{e,j-1}
\right]
\leq
\sigma^2.
\end{equation}
\end{assumption}
This conditional unbiasedness and bounded-noise model is the standard stochastic-gradient oracle used in local-SGD/FedAvg convergence analyses~\citep{stich2019local, khaled2020local, reddi2021adaptive}.

\begin{assumption}[Round-Wise Gradient Dissimilarity]
\label{assump:heterogeneity}
There exists \(\zeta \geq 0\) such that for every round \(e\) and every \(\theta\),
\begin{equation}
\Xi_e(\theta)
\triangleq
\sum_{i=1}^N w_i
\left\|
\nabla F_i^e(\theta) - \nabla F^e(\theta)
\right\|_2^2
\leq
\zeta^2.
\label{eq:round-wise-dissimilarity}
\end{equation}
\end{assumption}
This is the standard bounded gradient-dissimilarity condition used to quantify non-i.i.d. client drift in federated optimization~\citep{li2020fedprox, khaled2020local, karimireddy2020scaffold, reddi2021adaptive}. Here it is imposed on the round-frozen ForgeVLA objectives \(F_i^e\), so \(\zeta^2\) is algorithm-dependent and captures any reduction in cross-client disagreement induced by the shared bank.

\begin{assumption}[Unbiased Partial Participation]
\label{assump:partial-participation}
Conditioned on \(\mathcal{F}_e\), the participation indicators \(\{\xi_i^e\}_{i=1}^N\) introduced above are mutually independent and are also independent of the local mini-batch noise used to generate the coupled full-participation local updates \(\{g_i^e\}_{i=1}^N\). Consequently, the sampled server weights in Eq.~\eqref{eq:sampled-weights-anchor} satisfy
\begin{equation}
\mathbb{E}\!\left[\tilde{w}_i^e \mid \mathcal{F}_e\right] = w_i.
\end{equation}
Define the participation inflation factor
\begin{equation}
\Gamma_{\pi}
\triangleq
\max_{0 \leq e \leq E-1}
\max_{1 \leq i \leq N}
w_i \frac{1-p_i^e}{p_i^e}.
\label{eq:participation-inflation}
\end{equation}
\end{assumption}
Unbiased client sampling with bounded participation inflation is the partial-participation analogue of the client-sampling assumptions used in federated optimization analyses~\citep{karimireddy2020scaffold, reddi2021adaptive}.

\begin{lemma}[Anchored Adaptive-Aggregation Perturbation]
\label{lem:ag-perturbation}
Let \(A^e(\theta)\) denote the first quadratic term in Eq.~\eqref{eq:ag-analysis-update}. Suppose that for some \(G_{\mathrm{AG}} \geq 0\),
\begin{equation}
\mathbb{E}\!\left[
\|\nabla A^e(\tilde{\theta}^{e+1})\|_2^2
\middle|
\mathcal{F}_e
\right]
\leq
G_{\mathrm{AG}}^2\eta^2P^2.
\label{eq:ag-anchor-gradient}
\end{equation}
Then the solution of Eq.~\eqref{eq:ag-analysis-update}, with \(\alpha_{\mathrm{AG}}=\lambda_{\mathrm{AG}}/(\eta P)\), satisfies
\begin{equation}
\mathbb{E}\!\left[
\|\theta^{e+1}-\tilde{\theta}^{e+1}\|_2^2
\middle|
\mathcal{F}_e
\right]
\leq
\frac{G_{\mathrm{AG}}^2}{4\lambda_{\mathrm{AG}}^2}\eta^4P^4.
\label{eq:ag-anchor-distance}
\end{equation}
Consequently, by \(L\)-smoothness, Young's inequality, and the local-SGD moment bound in Eq.~\eqref{eq:local-moment}, Eq.~\eqref{eq:ag-perturb-control} holds for constants \(\rho_{\mathrm{AG}}\) and \(C_{\mathrm{AG}}\) independent of \(E\).
\end{lemma}
\begin{proof}
The function \(A^e\) is convex quadratic in \(\theta\). With \(\delta^e=\theta^{e+1}-\tilde{\theta}^{e+1}\), first-order optimality gives \(\nabla A^e(\theta^{e+1})+2\alpha_{\mathrm{AG}}\delta^e=0\). Monotonicity of \(\nabla A^e\) implies \(2\alpha_{\mathrm{AG}}\|\delta^e\|_2 \leq \|\nabla A^e(\tilde{\theta}^{e+1})\|_2\), which gives Eq.~\eqref{eq:ag-anchor-distance} after conditioning. Applying smoothness to \(F^e(\theta^{e+1})-F^e(\tilde{\theta}^{e+1})\), replacing \(\nabla F^e(\tilde{\theta}^{e+1})\) by \(\nabla F^e(\theta^e)\) plus an \(L\|\tilde{\theta}^{e+1}-\theta^e\|_2\) term, and using Eq.~\eqref{eq:local-moment} yields Eq.~\eqref{eq:ag-perturb-control}.
\end{proof}

\begin{assumption}[ForgeVLA Objective-Level Perturbation Control]
\label{assump:forgevla-regularity}
Let \(\tilde{\theta}^{e+1}\) denote the sampled FedAvg reference point defined in Eq.~\eqref{eq:sampled-weights-anchor}. There exist deterministic constants \(\rho_{\mathrm{AG}}, C_{\mathrm{AG}}, C_B \geq 0\), with \(\rho_{\mathrm{AG}}\) a sufficiently small universal constant, such that, for every round \(e\),
\begin{equation}
\mathbb{E}\!\left[
F^e(\theta^{e+1}) - F^e(\tilde{\theta}^{e+1})
\middle|
\mathcal{F}_e
\right]
\leq
\rho_{\mathrm{AG}}\eta P\|\nabla F^e(\theta^e)\|_2^2
+
C_{\mathrm{AG}}\eta^3 P^3,
\label{eq:ag-perturb-control}
\end{equation}
\begin{equation}
\mathbb{E}\!\left[
F^{e+1}(\theta^{e+1}) - F^e(\theta^{e+1})
\middle|
\mathcal{F}_e
\right]
\leq
C_B \eta^2 P.
\label{eq:objective-drift-control}
\end{equation}
\end{assumption}
This assumption isolates the two ForgeVLA-specific deviations from vanilla FedAvg at the objective level: adaptive aggregation contributes the perturbation term \(C_{\mathrm{AG}}\eta^3P^3\), and bank evolution contributes the drift term \(C_B\eta^2 P\). Lemma~\ref{lem:ag-perturbation} gives a sufficient anchored-aggregation condition for Eq.~\eqref{eq:ag-perturb-control}. Lemma~\ref{lem:bank-variation} shows that Eq.~\eqref{eq:objective-drift-control} follows from the bank-variation bound in Eq.~\eqref{eq:bank-variation}, with \(C_B=\sqrt{2}\lambda_{\mathrm{CP}}C_U/\tau\). The theorem only requires the objective-level bounds.

\textbf{Interpretation.}
The role of Assumption~\ref{assump:forgevla-regularity} is to capture, in a single unified statement, the two algorithm-specific perturbations that ForgeVLA introduces beyond standard local SGD.
The first term (Eq.~\eqref{eq:ag-perturb-control}) bounds how much the adaptive-aggregation step can worsen the current-round objective relative to the sampled FedAvg baseline $\tilde{\theta}^{e+1}$.
Because the adaptive-aggregation objective (Eq.~\eqref{eq:ag-analysis-update}) is anchored to $\tilde{\theta}^{e+1}$, Lemma~\ref{lem:ag-perturbation} reduces this requirement to a moment bound on the projection-gradient at the anchor.
The second term (Eq.~\eqref{eq:objective-drift-control}) bounds the inter-round objective change due to bank evolution; Lemma~\ref{lem:bank-variation} reduces it to bounded bank path length.
Thus, ForgeVLA preserves the FedAvg/local-SGD convergence order under standard stochastic optimization assumptions plus bounded adaptive-aggregation and bank-variation perturbations.

For the optional comparison to a fixed bank, define
\begin{equation}
F_U(\theta)
\triangleq
\sum_{i=1}^N w_i
\left(
\mathcal{L}_i(\theta)
+
\mathcal{L}_{\mathrm{CP},i}(\theta; U)
\right),
\end{equation}
and, when needed, assume the bank-to-gradient continuity condition
\begin{equation}
\|\nabla F_U(\theta) - \nabla F_{U'}(\theta)\|_2
\leq
L_U \|U - U'\|_F
\qquad
\text{for all } \theta \text{ and all banks } U,U',
\label{eq:bank-gradient-lipschitz}
\end{equation}
for some deterministic constant \(L_U \geq 0\).

\begin{theorem}[FedAvg-Order Stationarity of ForgeVLA]
\label{thm:convergence}
Assume Assumptions~\ref{assump:smooth}, \ref{assump:lower-bound}, \ref{assump:variance}, \ref{assump:heterogeneity}, \ref{assump:partial-participation}, and \ref{assump:forgevla-regularity}, together with \(\mathbb{E}[F^0(\theta^0)] < \infty\), the standard local-SGD small-step condition \(L \eta P \leq 1\), and the participation-stability condition \(L \eta P \Gamma_{\pi} \leq c_0\) for a sufficiently small universal constant \(c_0 > 0\).
Then there exist universal constants \(c_{\sigma}, c_{\mathrm{het}}, c_{\pi} > 0\) such that, with
\begin{equation}
\begin{aligned}
\mathcal{R}_E
\triangleq\;&
\frac{8\left(\mathbb{E}[F^0(\theta^0)] - F_{\mathrm{lb}}\right)}{\eta P E}
\;+\;
c_{\sigma} L \eta \sigma^2
\;+\;
c_{\mathrm{het}} L^2 \eta^2 P (P-1)\zeta^2 \\
&+\;
c_{\pi} L \eta P \Gamma_{\pi}(\sigma^2 + \zeta^2)
\;+\;
8 C_{\mathrm{AG}}\eta^2 P^2
\;+\;
8 C_B\eta,
\end{aligned}
\label{eq:integrated-rate}
\end{equation}
the following main bound holds:
\begin{equation}
\frac{1}{E}\sum_{e=0}^{E-1}\mathbb{E}\|\nabla F^e(\theta^e)\|_2^2
\leq
\mathcal{R}_E,
\label{eq:main-stationarity-bound}
\end{equation}
If, in addition, Eq.~\eqref{eq:bank-gradient-lipschitz} holds, then for any reference bank \(U^\star\),
\begin{equation}
\frac{1}{E}\sum_{e=0}^{E-1}\mathbb{E}\|\nabla F_{U^\star}(\theta^e)\|_2^2
\leq
2\mathcal{R}_E
+
\frac{2L_U^2}{E}\sum_{e=0}^{E-1}\mathbb{E}\|U^e - U^\star\|_F^2.
\label{eq:fixed-bank-transfer}
\end{equation}
Consequently, if the bank sequence stabilizes in the Ces\`aro sense around some \(U^\star\), then this additional continuity condition transfers the round-wise stationarity guarantee directly to the fixed objective \(F_{U^\star}\). Moreover, under the stated bounded-perturbation conditions, fixed \(P\), bounded \(\Gamma_{\pi}\), and \(E\)-independent \(C_{\mathrm{AG}}, C_B\), choosing \(\eta = \Theta(E^{-1/2})\) gives \(\mathcal{R}_E = \mathcal{O}(E^{-1/2})\), matching the standard nonconvex local-SGD/FedAvg order. If Eq.~\eqref{eq:bank-gradient-lipschitz} also holds with \(E\)-independent \(L_U\) and
\(
\frac{1}{E}\sum_{e=0}^{E-1}\mathbb{E}\|U^e - U^\star\|_F^2 = \mathcal{O}(E^{-1/2}),
\)
then the same order transfers to \(F_{U^\star}\). When adaptive aggregation and bank evolution are absent, \textit{i.e.}, \(C_{\mathrm{AG}} = C_B = 0\) and \(F_i^e(\theta) \equiv \mathcal{L}_i(\theta)\), Eq.~\eqref{eq:main-stationarity-bound} reduces to the usual heterogeneity- and participation-aware FedAvg/local-SGD guarantee up to universal constants.
\end{theorem}

\begin{proof}
Conditioned on \(\mathcal{F}_e\), couple the sampled round with the hypothetical full-participation local-SGD recursion driven by the same mini-batch draws. The standard nonconvex local-SGD descent estimate under bounded conditional noise and bounded gradient dissimilarity yields~\citep{li2020fedprox, khaled2020local, karimireddy2020scaffold, reddi2021adaptive}
\begin{equation}
\mathbb{E}\!\left[F^e(\bar{\theta}^{e+1}) \mid \mathcal{F}_e\right]
\leq
F^e(\theta^e)
- \frac{\eta P}{4}\|\nabla F^e(\theta^e)\|_2^2
+
c_{\sigma}' L \eta^2 P \sigma^2
+
c_{\mathrm{het}}' L^2 \eta^3 P^2 (P-1)\Xi_e(\theta^e),
\label{eq:frozen-descent}
\end{equation}
for universal constants \(c_{\sigma}', c_{\mathrm{het}}' > 0\).

Write the sampled displacement as
\(
\tilde{g}^e = \sum_{i=1}^N \tilde{w}_i^e g_i^e = \bar{g}^e + \varepsilon_{\pi}^e
\),
where
\(
\bar{g}^e = \sum_{i=1}^N w_i g_i^e
\).
Define
\(
\varepsilon_{\pi}^e
=
\sum_{i=1}^N
w_i\left(\frac{\xi_i^e}{p_i^e} - 1\right) g_i^e.
\)
By Assumption~\ref{assump:partial-participation}, conditioned on \(\mathcal{F}_e\) and \(\{g_i^e\}_{i=1}^N\), the coefficients \(w_i(\xi_i^e/p_i^e - 1)\) are mean-zero and mutually independent, hence the cross terms vanish and
\(
\mathbb{E}[\varepsilon_{\pi}^e \mid \mathcal{F}_e] = 0
\)
and
\begin{equation}
\mathbb{E}\!\left[\|\varepsilon_{\pi}^e\|_2^2 \middle| \mathcal{F}_e, \{g_i^e\}_{i=1}^N\right]
=
\sum_{i=1}^N
w_i^2 \frac{1-p_i^e}{p_i^e}\|g_i^e\|_2^2.
\label{eq:participation-noise-exact}
\end{equation}
Using \(w_i^2 \frac{1-p_i^e}{p_i^e} \leq \Gamma_{\pi} w_i\) for every \(i,e\) and then averaging over the local mini-batch noise yields
\begin{equation}
\mathbb{E}\!\left[\|\varepsilon_{\pi}^e\|_2^2 \mid \mathcal{F}_e\right]
\leq
\Gamma_{\pi}
\sum_{i=1}^N w_i \mathbb{E}\!\left[\|g_i^e\|_2^2 \mid \mathcal{F}_e\right].
\label{eq:participation-noise}
\end{equation}
The same recursion also gives
\begin{equation}
\sum_{i=1}^N w_i \mathbb{E}\!\left[\|g_i^e\|_2^2 \mid \mathcal{F}_e\right]
\leq
c_g \eta^2 P^2
\left(
\|\nabla F^e(\theta^e)\|_2^2 + \sigma^2 + \zeta^2
\right)
\label{eq:local-moment}
\end{equation}
for a universal constant \(c_g > 0\).
Applying \(L\)-smoothness to \(F^e(\tilde{\theta}^{e+1})\) and using Eqs.~\eqref{eq:frozen-descent}--\eqref{eq:local-moment} yields
\begin{equation}
\begin{aligned}
\mathbb{E}\!\left[F^e(\tilde{\theta}^{e+1}) \mid \mathcal{F}_e\right]
\leq\;&
F^e(\theta^e)
- \frac{\eta P}{4}\|\nabla F^e(\theta^e)\|_2^2
+
c_{\sigma}' L \eta^2 P \sigma^2 \\
&+
c_{\mathrm{het}}' L^2 \eta^3 P^2 (P-1)\Xi_e(\theta^e)
+
c_{\pi}' L \eta^2 P^2 \Gamma_{\pi}
\left(
\|\nabla F^e(\theta^e)\|_2^2 + \sigma^2 + \zeta^2
\right),
\end{aligned}
\label{eq:sampled-descent}
\end{equation}
for a universal constant \(c_{\pi}' > 0\).

Assumption~\ref{assump:forgevla-regularity} gives
\begin{equation}
\mathbb{E}\!\left[F^e(\theta^{e+1}) \mid \mathcal{F}_e\right]
\leq
\mathbb{E}\!\left[F^e(\tilde{\theta}^{e+1}) \mid \mathcal{F}_e\right]
+
\rho_{\mathrm{AG}}\eta P\|\nabla F^e(\theta^e)\|_2^2
+
C_{\mathrm{AG}}\eta^3 P^3,
\label{eq:ag-perturbed-descent}
\end{equation}
while Eq.~\eqref{eq:objective-drift-control} gives
\begin{equation}
\mathbb{E}\!\left[F^{e+1}(\theta^{e+1}) \mid \mathcal{F}_e\right]
\leq
\mathbb{E}\!\left[F^e(\theta^{e+1}) \mid \mathcal{F}_e\right]
+
C_B\eta^2 P.
\label{eq:objective-drift-bound}
\end{equation}

Combining Eqs.~\eqref{eq:sampled-descent}--\eqref{eq:objective-drift-bound}, using \(\Xi_e(\theta^e) \leq \zeta^2\), and absorbing the gradient-dependent perturbations under the sufficiently small universal constants \(\rho_{\mathrm{AG}}\) and \(L \eta P \Gamma_{\pi} \leq c_0\) yields
\begin{equation}
\begin{aligned}
\mathbb{E}\!\left[F^{e+1}(\theta^{e+1}) \mid \mathcal{F}_e\right]
\leq\;&
F^e(\theta^e)
- \frac{\eta P}{8}\|\nabla F^e(\theta^e)\|_2^2
+
c_{\sigma}'' L \eta^2 P \sigma^2 \\
&+
c_{\mathrm{het}}'' L^2 \eta^3 P^2 (P-1)\zeta^2
+
c_{\pi}'' L \eta^2 P^2 \Gamma_{\pi}(\sigma^2 + \zeta^2) \\
&+
C_{\mathrm{AG}}\eta^3 P^3
+
C_B\eta^2 P,
\end{aligned}
\end{equation}
for universal constants \(c_{\sigma}'', c_{\mathrm{het}}'', c_{\pi}'' > 0\). Summing over \(e=0,\dots,E-1\), telescoping, using Assumption~\ref{assump:lower-bound}, and dividing by \(\eta P E / 8\) gives Eq.~\eqref{eq:main-stationarity-bound} after absorbing constants into \(c_{\sigma}, c_{\mathrm{het}}, c_{\pi}\) and Eq.~\eqref{eq:integrated-rate}.

For the fixed-objective part, use \(F^e(\theta) = F_{U^e}(\theta)\) and Eq.~\eqref{eq:bank-gradient-lipschitz}:
\[
\|\nabla F_{U^\star}(\theta^e)\|_2
\leq
\|\nabla F^e(\theta^e)\|_2
+
\|\nabla F_{U^\star}(\theta^e) - \nabla F_{U^e}(\theta^e)\|_2
\leq
\|\nabla F^e(\theta^e)\|_2 + L_U\|U^e - U^\star\|_F.
\]
Squaring, averaging over \(e\), and using Eq.~\eqref{eq:main-stationarity-bound} gives Eq.~\eqref{eq:fixed-bank-transfer}. This transfer from round-wise stationarity to a fixed objective is standard for optimization with evolving surrogates or time-varying losses~\citep{mairal2013surrogate, razaviyayn2013bsum, lin2023timevarying}.
\end{proof}

\section{Empirical Convergence}

\begin{figure*}[tb]
    \centering
    \includegraphics[width=0.75\textwidth]{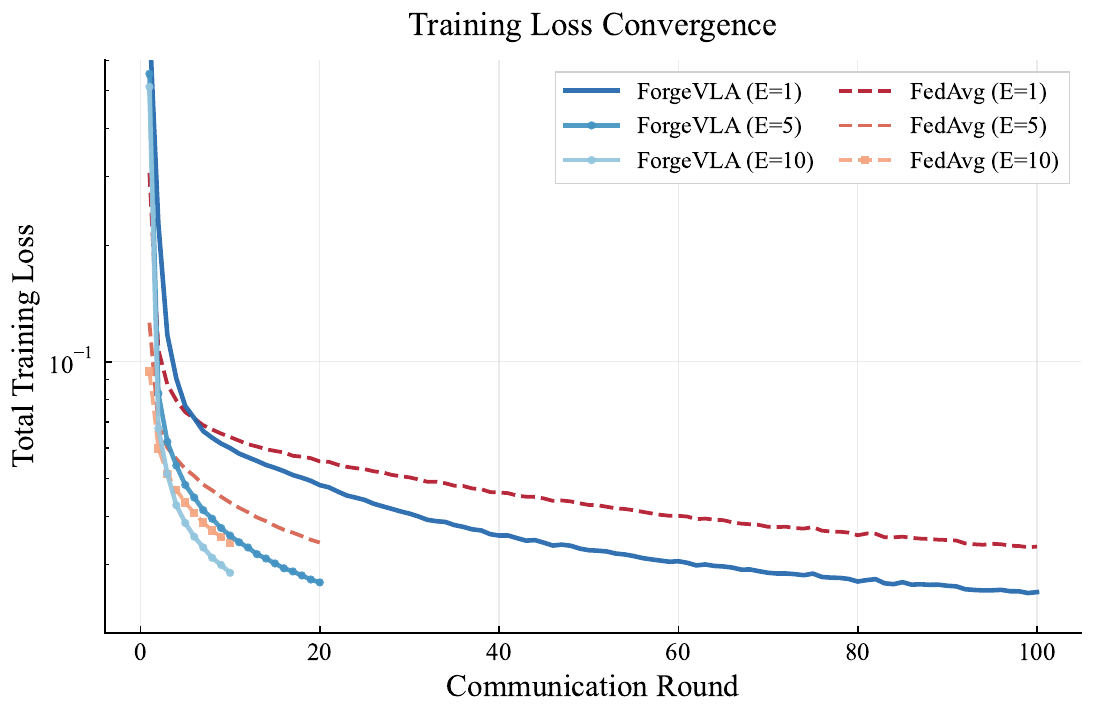}
    \caption{
        The training loss curves of ForgeVLA and FedAvg.
    }
    \label{fig:training-loss}
\vspace{-10pt}
\end{figure*}

Figure~\ref{fig:training-loss} shows the training loss curves of ForgeVLA and FedAvg on the LIBERO-Goal benchmark under different training configurations.
ForgeVLA achieves comparable training loss to FedAvg across all configurations, empirically validating the convergence stability of our method and complementing the theoretical analysis in Section~\ref{sec:convergence}.

\section{Additional Implementation Details}
\label{sec:implementation-details}

We use InternVLA-M1~\citep{chen2025internvla} as the backbone VLA, adapting the VLM encoder with LoRA~\citep{hu2022lora} ($r{=}32$, all linear layers; 128.10M/3882.72M trainable parameters) and fully fine-tuning the action decoder.
The embodied instruction classifier shares the same VLM backbone with an attention-pooling head and is fine-tuned with LoRA.
We simulate $N{=}10$ clients with full participation, $s{=}3$ tasks per client (non-i.i.d. task-level partitioning), $E{=}20$ communication rounds, and $P{=}5$ local epochs per round at batch size 32.
We use AdamW (zero weight decay) with learning rates $10^{-5}$ (encoder) and $10^{-4}$ (decoder), both with linear warmup.
ForgeVLA hyperparameters: $\alpha_{\text{CP}}{=}0.2$, $\tau{=}0.07$, $\alpha_{\text{AG}}{=}0.1$.
We evaluate with 50 rollouts per task across 10 tasks per LIBERO suite and report success rate and $\text{Pass}@50$.
All experiments use CUDA 12.8, Python 3.10.19, and PyTorch 2.6.0; baselines follow their original hyperparameters.

\section{Additional Evaluation Results}
We present additional evaluation results to analyze the training efficiency and robustness of ForgeVLA under different local optimization budgets $P$ in Table~\ref{tab:additional-evaluation}.
When trained with $P=10$ local optimization iterations per communication round (implemented as local epochs), ForgeVLA achieves performance that surpasses FedAvg with $P=5$ on LIBERO-Goal and is comparable to FedAvg with $P=1$ across all benchmarks.
In contrast, vanilla FedAvg suffers from significant performance degradation as the number of local optimization iterations increases, primarily due to exacerbated client drift and vision-language feature collapse.
By leveraging the proposed contrastive planning loss and adaptive aggregation strategy, ForgeVLA effectively alleviates representation collapse and conflicting local updates, maintaining stable and high-performance learning even with intensive local optimization.

Notably, in the $E{=}100, P{=}1$ configuration, ForgeVLA achieves 81.2\% on LIBERO-Goal, which slightly exceeds the centralized upper bound (75.8\%).
We attribute this to the contrastive planning loss, which explicitly encourages inter-task discriminability in the embedding space, a regularization effect absent from the standard centralized training objective.
We emphasize that this result is based on a single seed and that reporting mean and standard deviation over multiple runs would provide a more reliable comparison; we leave such variance analysis for future work.

\section{Privacy Implication}

ForgeVLA follows the standard FL data-retention principle: raw data never leaves client devices, and clients exchange only model updates with the server, matching the vanilla FedAvg workflow.
We consider the standard honest-but-curious server threat model commonly adopted in federated learning~\citep{kairouz2021advances}: the server faithfully executes the protocol but may attempt to infer client information from received updates.
Under this model, ForgeVLA exposes essentially the same information as FedAvg, with two additional transmissions: (i) a one-time server-to-client broadcast of the embodied instruction classifier trained on public data (which contains no client-private information) and (ii) per-task mean representations.
These per-task prototypes are aggregated statistics that have been argued to carry lower reconstruction risk than full gradient updates~\citep{tan2022fedproto}, though we note they are not formally private.
Consequently, existing privacy-enhancing techniques for FL, such as secure aggregation~\citep{bonawitz2017practical} and differential privacy~\citep{abadi2016deep}, can be applied on top of ForgeVLA without modifying the core algorithm to provide stronger formal guarantees if required.

\section{Limitations}

While ForgeVLA demonstrates strong performance across all benchmarks and validates real-world transferability, we acknowledge several limitations.
First, the embodied instruction classifier maps vision--action pairs to a fixed set of $M$ predefined language instructions. This closed-set design works well for structured industrial and service robotics scenarios but cannot handle open-ended task descriptions.
Second, ForgeVLA follows the standard honest-but-curious threat model and lacks formal privacy mechanisms. While transmitted prototypes have lower reconstruction risk than full gradients, they are not formally private and remain vulnerable to stronger gradient inversion attacks.
Thus, extending to open-vocabulary instruction generation and integrating formal privacy guarantees with rigorous privacy-utility trade-off analysis are important directions for future work.

\begin{table}[htbp]
  \def\arraystretch{1.3}
  \addtolength{\tabcolsep}{0pt}
  \centering
  \tablestyle{3pt}{1.0}
  \caption{The performance comparisons between different methods. \textbf{bold} marks the best-performing results.}
  \resizebox{\linewidth}{!}{%
    \begin{tabular}{cccccc}
    \toprule
    Methods & Datasets & \# Params (M) & \# Trainable Params (M) & Success Rate (\%) & $\text{Pass}@50$ (\%) \\
    \midrule
    \multicolumn{6}{c}{$E=100$; $P=1$} \\
    \midrule
    Centralized & LIBERO-Goal & 3882.724 & 128.101 & 75.8  & 100 \\
    FedAvg & LIBERO-Goal & 3882.724 & 128.101 & 57.6  & 90 \\
    ForgeVLA & LIBERO-Goal & 3882.724 & 128.101 & \textbf{81.2}\textcolor{ForestGreen}{\scriptsize $^\uparrow$23.6\%}  & 100 \\
    \midrule
    Centralized & LIBERO-Object & 3882.724 & 128.101 & 98.8  & 100 \\
    FedAvg & LIBERO-Object & 3882.724 & 128.101 & 98.2  & 100 \\
    ForgeVLA & LIBERO-Object & 3882.724 & 128.101 & \textbf{99.4}\textcolor{ForestGreen}{\scriptsize $^\uparrow$1.2\%}  & 100 \\
    \midrule
    Centralized & LIBERO-Spatial & 3882.724 & 128.101 & 85.8  & 100 \\
    FedAvg & LIBERO-Spatial & 3882.724 & 128.101 & 80.6  & 100 \\
    ForgeVLA & LIBERO-Spatial & 3882.724 & 128.101 & \textbf{85.4}\textcolor{ForestGreen}{\scriptsize $^\uparrow$4.8\%}  & 100 \\
    \midrule
    Centralized & LIBERO-10 & 3882.724 & 128.101 & 79    & 100 \\
    FedAvg & LIBERO-10 & 3882.724 & 128.101 & 71.8  & 100 \\
    ForgeVLA & LIBERO-10 & 3882.724 & 128.101 & \textbf{78.8}\textcolor{ForestGreen}{\scriptsize $^\uparrow$7.0\%}  & 100 \\
    \midrule
    \multicolumn{6}{c}{$E=10$; $P=10$} \\
    \midrule
    Centralized & LIBERO-Goal & 3882.724 & 128.101 & 75.8  & 100 \\
    FedAvg & LIBERO-Goal & 3882.724 & 128.101 & 17.6  & 60 \\
    ForgeVLA & LIBERO-Goal & 3882.724 & 128.101 & \textbf{41.6}\textcolor{ForestGreen}{\scriptsize $^\uparrow$24.0\%}  & 100 \\
    \midrule
    Centralized & LIBERO-Object & 3882.724 & 128.101 & 98.8  & 100 \\
    FedAvg & LIBERO-Object & 3882.724 & 128.101 & 94    & 100 \\
    ForgeVLA & LIBERO-Object & 3882.724 & 128.101 & \textbf{95.6}\textcolor{ForestGreen}{\scriptsize $^\uparrow$1.6\%}  & 100 \\
    \midrule
    Centralized & LIBERO-Spatial & 3882.724 & 128.101 & 85.8  & 100 \\
    FedAvg & LIBERO-Spatial & 3882.724 & 128.101 & 50.8  & 90 \\
    ForgeVLA & LIBERO-Spatial & 3882.724 & 128.101 & \textbf{68.2}\textcolor{ForestGreen}{\scriptsize $^\uparrow$17.4\%}  & 90 \\
    \midrule
    Centralized & LIBERO-10 & 3882.724 & 128.101 & 79    & 100 \\
    FedAvg & LIBERO-10 & 3882.724 & 128.101 & 23.4  & 80 \\
    ForgeVLA & LIBERO-10 & 3882.724 & 128.101 & \textbf{36.4}\textcolor{ForestGreen}{\scriptsize $^\uparrow$13.0\%}  & 90 \\
    \bottomrule
    \end{tabular}%
    }
  \label{tab:additional-evaluation}%
\end{table}%

\section{Experiments Compute Resources}

All experiments are conducted on NVIDIA H800 GPUs. Each H800 is equipped with 128 GB CPU memory and a 16-core CPU. For each ForgeVLA experiment, the training budget amounts to two H800 GPUs running for 48 hours in total.

\section{Border Impacts}

Our work enables broader utilization of diverse vision-language data in federated VLA learning, and can facilitate the development of generalizable embodied and robotic VLA models. 
It brings positive impacts to downstream research including multimodal understanding, robotic manipulation, and privacy-preserving distributed learning. 
We do not identify any foreseeable negative societal or ethical impacts arising from this work.

\section{Declaration of LLM Usage}

LLMs were used solely for linguistic polishing and sentence refinement in this manuscript. 
All LLM-generated content has been fully reviewed and revised by the authors, who retain full responsibility for all technical claims and conclusions.

\end{document}